\documentclass[pdflatex,sn-mathphys-num]{sn-jnl}

\usepackage{graphicx}%
\usepackage{multirow}%
\usepackage{amsmath,amssymb,amsfonts,mathtools}%
\usepackage{amsthm}%
\usepackage{mathrsfs}%
\usepackage[title]{appendix}%
\usepackage{xcolor}%
\usepackage{textcomp}%
\usepackage{manyfoot}%
\usepackage{booktabs}%
\usepackage{algorithm}%
\usepackage{algorithmicx}%
\usepackage{algpseudocode}%
\usepackage{listings}%
\usepackage{subcaption}
\usepackage{bm}
\usepackage[nameinlink, capitalize]{cleveref}
\usepackage{enumitem}

\theoremstyle{thmstyleone}%
\theoremstyle{thmstyletwo}%

\theoremstyle{thmstylethree}%

\graphicspath{{images/}}

\newtheorem{theo}{Theorem}[section]
\newtheorem{lem}[theo]{Lemma}

\begin{document}

\title[MG-SpaIR: Multi-grade Sparse-guided Implicit Representation for Training-Data-Free Image Restoration]{MG-SpaIR: Multi-grade Sparse-guided Implicit Representation for Training-Data-Free Image Restoration}


\author*[1]{\fnm{Jianmin} \sur{Liao}}\email{jliao21@syr.edu}
\equalcont{These authors contributed equally to this work.}

\author[2]{\fnm{Lei} \sur{Huang}}\email{lhuan001@odu.edu}
\equalcont{These authors contributed equally to this work.}

\author[3]{\fnm{Ronglong} \sur{Fang}}\email{fangr1@mskcc.org}

\author[4]{\fnm{Ashley} \sur{Prater-Bennette}}\email{Ashley.Prater-Bennette@us.af.mil}

\author[1]{\fnm{Lixin} \sur{Shen}}\email{lshen03@syr.edu}

\author[2]{\fnm{Yuesheng} \sur{Xu}}\email{y1xu@odu.edu}

\affil*[1]{\orgdiv{Department of Mathematics}, \orgname{Syracuse University}, \orgaddress{\street{215 Carnegie Building}, \city{Syracuse}, \postcode{13210}, \state{NY}, \country{USA}}}

\affil[2]{\orgdiv{Department of Mathematics \& Statistics}, \orgname{Old Dominion University}, \orgaddress{\street{{
2300 Engineering \& Computational Sciences Building}}, \city{Norfolk}, \postcode{23529}, \state{VA}, \country{USA}}}

\affil[3]{\orgdiv{Department of Medical Physical}, \orgname{Memorial Sloan Kettering Cancer Center}, \orgaddress{\street{1250 First Avenue}, \city{New York}, \postcode{10065}, \state{NY}, \country{USA}}}

\affil[4]{\orgname{Air Force Research Laboratory}, \orgaddress{\street{525 Brooks Road}, \city{Rome}, \postcode{13441}, \state{NY}, \country{USA}}}


\abstract{
MG-SpaIR is a training-data-free framework for restoring a clean image from a single observation corrupted by a mixture of blur, downsampling, noise, and missing pixels. Building on implicit neural representations (INRs), we introduce a multi-grade coarse-to-fine residual hierarchy that progressively refines the reconstruction across resolution grades, improving representational fidelity and mitigating spectral limitations. To stabilize reconstruction optimization and suppress INR-induced artifacts, we further propose an explicit sparse proximal regularization (e.g., $\ell_0$-type) applied directly in the high-resolution image domain, which discourages spurious high-frequency patterns while preserving sharp structures. The resulting optimization is solved efficiently via a multi-grade proximal alternating scheme, and we establish convergence guarantees for the associated updates under standard regularity conditions. Experiments on mixed-degradation benchmarks demonstrate that MG-SpaIR consistently outperforms strong training-data-free baselines such as Deep Image Prior, providing a stable, interpretable, and data-efficient alternative to conventional learning-based restoration methods.
}




\maketitle

\begin{center}
\small\textit{A revised version of this manuscript has been accepted for publication in the \textit{Journal of Mathematical Imaging and Vision}, Collection: Special Issue on Mathematics of Imaging and Machine Learning.}
\end{center}

\section{Introduction}\label{sec:intro}
We study the restoration of a clean image from a single observation corrupted by a mixture of degradations—blur, downsampling, noise, and missing pixels—reflecting common failure modes in real imaging systems. Because many clean images can explain the same degraded observation, image restoration is inherently ill-posed, making effective priors essential for accurate recovery.

Classical restoration methods rely on explicit priors, most notably sparsity-based regularization such as total variation~\cite{rudin1992nonlinear} and non-local self-similarity as in BM3D~\cite{dabov2007image}. While effective, these priors are typically defined on discrete pixel grids and can struggle to represent complex long-range dependencies and fine continuous structures (e.g., hair, fur, and thin lines).




Neural networks have been explored as implicit image priors for their ability to capture natural image statistics~\cite{ulyanov2018deep}. Implicit Neural Representations (INRs)~\cite{sitzmann2020implicit, muller2022instant, saragadam2023wire, ramasinghe2022beyond, hao2022implicit, pmlr-v235-pal24a,zhao2025adaptive} model images as continuous coordinate-to-color mappings, naturally recovering missing pixels and fine details, making them effective for training-data-free tasks like super-resolution and inpainting.


Despite their appeal, standard INRs for restoration face two fundamental challenges. First, the inherent spectral bias \cite{tancik2020fourier,zhao2025adaptive} impedes the accurate representation of high-frequency structures, leading to blurry textures and poor detail recovery. Second, reliance solely on the implicit architectural prior often destabilizes the training-data-free optimization, injecting spurious high-frequency artifacts (e.g., ringing, checkerboards). This artifact tendency is evident in prior works (e.g., Fig.~7 in \cite{saragadam2023wire}, Supp. Fig.~8 of~\cite{sitzmann2020implicit}) and is confirmed by our experiments (\cref{fig:reg_vs_no_reg inpainting,fig:reg_vs_no_reg superres}). 
To the best of our knowledge, we are among the first to explicitly identify the tendency of INRs to generate artifacts and to propose an explicit regularization solution.

To address these challenges, we introduce the \textbf{MG-SpaIR (Multi-Grade Sparse-guided Implicit Representation)}, which combines a coarse-to-fine representation strategy with explicit sparse regularization: 
\begin{itemize}
\item \textbf{Multi-grade INR to mitigate spectral limitations.}
We propose a coarse-to-fine residual hierarchy that learns high-frequency content progressively across grades, improving detail recovery while remaining compute-efficient.

\item \textbf{Sparse proximal regularization for artifact suppression.}
We introduce an explicit $\ell_0$-type proximal prior imposed in the high-resolution image domain to stabilize reconstruction and reduce INR-induced artifacts.

\item \textbf{Convergence analysis.}
We provide convergence guarantees for the proposed optimization scheme under standard regularity assumptions. 

\item \textbf{Empirical validation.}
We demonstrate consistent improvements over strong training-data-free baselines (e.g., DIP) across mixed-degradation benchmarks.
\end{itemize}

An overview of the proposed multi-grade implicit representation framework is illustrated in \cref{fig:overview}\footnote{Unless otherwise noted, individual ground truth images used in figures (e.g., \cref{fig:overview}) are sourced from standard benchmarks or Wikimedia Commons (CC BY-SA 2.0).}. Each grade corresponds to a shallow INR optimized directly on the degraded input image. Through progressive refinement, the network’s representational ability improves across grades, capturing fine structures and recovering missing details. Although not depicted in the figure, sparse proximal regularization is applied in the high-resolution image domain to guide the optimization toward stable and faithful reconstructions.

\begin{figure*}[htbp!]
    \centering
    \includegraphics[width=1.0\columnwidth]{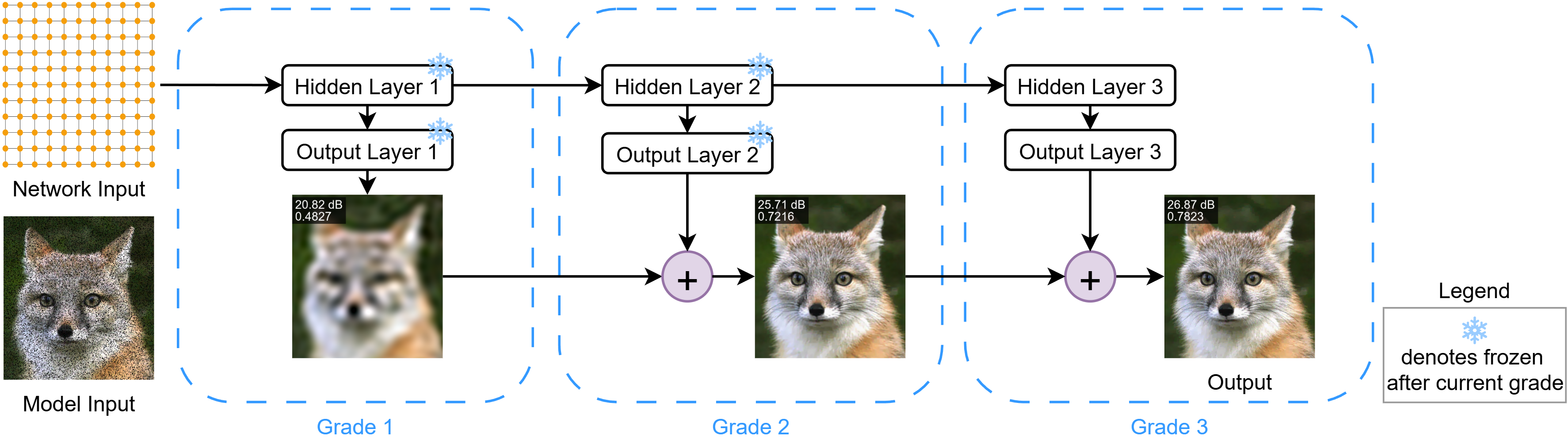}
    \caption{\textbf{Architecture Overview.} The model takes a degraded image as input, while the network input is a uniform spatial grid over $[-1,1]^2$. The restored image is implicitly encoded in the neural network parameters and guided by sparse regularization during optimization. Training proceeds in three successive grades: at each grade, the parameters from all previous grades are frozen, and only the current grade is updated. Each new grade learns from the residual of the previous reconstruction, progressively refining image details. The output of the final grade produces the restored image.}
 \label{fig:overview}
\end{figure*}

The remainder of this paper is organized as follows.
\Cref{sec:related_work} reviews related work on implicit neural representations and image restoration.
\Cref{sec:model} presents the proposed model formulation, including the multi-grade strategy, and details the MG-SpaIR regularization and the optimization algorithm for solving the resulting objective.
\Cref{section: Convergence Analysis} provides convergence analysis of the proposed scheme.
\Cref{sec:experiments} reports experimental results and ablation studies.
Finally, \cref{sec:conclusion} concludes the paper and discusses limitations and future directions.


\section{Related Work}\label{sec:related_work}
\noindent\textbf{Classical explicit priors.}\;
Classical restoration methods rely on explicit priors to constrain the solution space. 
Sparsity-based constraints, such as $\ell_1$/TV \cite{rudin1992nonlinear}, $\ell_0$ \cite{xu2023sparse,zeng2016convergent,wu2022inverting,bao2016image}, and nonconvex surrogates (SCAD \cite{fan2001variable}, MCP \cite{zhang2010nearly}), are typically solved via proximal algorithms. 
Non-local priors (e.g., NLM \cite{buades2005non}, BM3D \cite{dabov2007image}) utilize self-similarity to suppress noise but often incur prohibitive computational costs at high resolutions. 
While theoretically sound, these explicit priors struggle to model complex continuous structures compared to neural networks.

\noindent\textbf{Data-driven vs. implicit neural priors.}\;
Deep learning approaches have largely superseded classical methods, yet they face the trade-off between fidelity and perceptual quality.
Supervised models like SwinIR~\cite{liang2021swinir} and Restormer~\cite{Zamir2021Restormer} achieve state-of-the-art perceptual quality by learning statistical priors from massive datasets. However, this dependence often leads to \textit{hallucinated textures}—details that appear realistic but do not exist in the ground truth~\cite{bhadra2021hallucinations,doi:10.1073/pnas.1907377117}, rendering them unreliable for fidelity-critical applications.
Conversely, implicit neural priors like DIP~\cite{ulyanov2018deep} avoid dataset bias by exploiting the inductive bias of network architectures. However, without explicit constraints, the optimization landscape of these purely implicit methods is often treacherous, leading to instability and high-frequency artifacts as the network overfits to noise.

\noindent\textbf{Implicit neural representations.}\;
Implicit neural representations (INRs) model continuous signals via neural networks and are widely applied in geometry and imaging.
Early works such as DeepSDF \cite{Park2019DeepSDF} and Occupancy Networks \cite{Mescheder2019OccupancyNetworks} established implicit 3D modeling.
Subsequent advances addressed spectral bias through oscillatory bases (SIREN \cite{sitzmann2020implicit}, WIRE \cite{saragadam2023wire}, MFN \cite{Fathony2021MFN}, Residual MFN \cite{Shekarforoush2022ResidualMFN}) and input encodings \cite{tancik2020fourier,muller2022instant}.
Further control over frequency and scale was introduced via progressive or band-limited encodings \cite{Lindell2022BACON,hertz2021sape} and adaptive multiscale \cite{Martel2021ACORN} or meta-learned frameworks \cite{lee2021meta}.

\noindent\textbf{Spectral bias of neural networks.}\; 
Previous studies \cite{rahaman2019spectral,xu2019training,fang2024addressing,fang2025computational} have shown that neural networks preferentially learn low-frequency components while underrepresenting high-frequency details. This phenomenon, known as \textit{spectral bias}, limits their performance in applications requiring fine high-frequency reconstruction, such as image restoration.

\noindent\textbf{INRs for image restoration.}\;
Existing INR-based methods incorporating explicit priors for denoising or inpainting often rely on non-local self-similarity \cite{luo2024revisiting} or TV regularization \cite{luo2025neurtv}.
In contrast, we introduce (i) a multi-grade strategy to mitigate spectral bias and enhance the INR’s representational accuracy, and (ii) an explicit sparsity prior in the \emph{high-resolution} image domain, extending beyond TV/$\ell_1$ to $\ell_0$ and MCP regularizations to directly suppress high-resolution artifacts and noise.

\section{Proposed Model and Algorithm}\label{sec:model}
In this section, we present the MG-SpaIR optimization model. We first formulate a unified image restoration setting, then describe our multi-grade strategy for enhancing image representation, and finally introduce our restoration model along with its optimization algorithm.

Let $\mathbf{Y} \in \mathbb{R}^{H \times W \times C}$ denote a latent clean image with height $H$, width $W$, and $C$ channels. We adopt a unified degradation model expressed as:
\begin{equation}
    \tilde{\mathbf{Y}} = \mathbf{M} \! \odot \! \left( \mathcal{N} (\mathcal{S}_s (\mathcal{A}(\mathbf{Y}))) \right),
    \label{eq:degradation}
\end{equation}
where $\mathcal{A}$ models optical blur, simulating the loss of high-frequency details due to lens aberration or motion; $\mathcal{S}_s$ denotes downsampling by factor $s$, covering cases such as digital zoom-out, finite sensor resolution, or pipeline decimation; $\mathcal{N}$ represents additive noise from sources like photon shot noise or sensor readout; and $\mathbf{M}$ is a binary mask modeling missing-pixel degradation (e.g., scratches or random sensor failures). The element-wise multiplication $\odot$ applies the mask to the observed signal.

This unified formulation captures a wide range of real-world degradations—blur, noise, downsampling, and missing data. The goal of image restoration is to recover the clean image $\mathbf{Y}$ from its degraded observation $\tilde{\mathbf{Y}}$.


\subsection{Building a Strong Representational Foundation}\label{sec:INR}
We review implicit neural representations (INRs) for images, introduce a multi-grade extension of a plain (single-grade) INR, and discuss their benefits.

Rather than optimizing pixel values directly, an image can be represented by a neural network mapping 2D coordinates to pixel values. Let $\Omega=[-1,1]^2$ denote the normalized coordinate domain, $\mathbf{x} \in \Omega$ a 2D coordinate, and $C$ output channels (1 for grayscale, 3 for RGB). A \textit{single-grade INR} defines a continuous function $f_{\theta}: \Omega \to \mathbb{R}^C$, where $f_{\theta}(\mathbf{x})$ returns the pixel value at $\mathbf{x}$. For a coordinate grid $\mathbf{X} \in \Omega^{H \times W}$, the latent image is obtained by evaluating the network at each coordinate:
\begin{equation*}
\mathbf{Y} = f_{\theta}(\mathbf{X}), \quad
\text{that is, }\ \mathbf{Y}_{i,j} = f_{\theta}(\mathbf{X}_{i,j}).
\end{equation*}
This shifts the optimization from pixel space to network parameter space, leveraging the expressivity of neural networks to produce continuous, plausible image reconstructions.

\noindent\textbf{Definition of Single-Grade INR.}\label{def:SGINR}
A \textit{single-grade implicit neural representation} (SG-INR) 
is a network modeling a continuous mapping
$f_\theta : \Omega \subset \mathbb{R}^2 \to \mathbb{R}^C$,
where \(C\) denotes the number of channels. The output at \(\mathbf{x}\) is computed via a sequence of parameterized layers
$\{\phi^{(l)}_{\theta_l}\}_{l=1}^L$ and an output layer $\psi_{\theta'}$:
\begin{equation}
\begin{aligned}
&\mathbf{a}^{(0)} = \mathbf{x}\in\Omega\subset\mathbb{R}^2, &
\mathbf{a}^{(1)} = \phi^{(1)}_{\theta_1}(\mathbf{a}^{(0)}),
\\
&\mathbf{a}^{(l)} = \phi^{(l)}_{\theta_l}(\mathbf{a}^{(l-1)}),
&
f_\theta(\mathbf{x}) = \psi_{\theta'}(\mathbf{a}^{(L)}),
\end{aligned}
\end{equation}
where $\phi^{(1)}_{\theta_1}:\mathbb{R}^2\to\mathbb{R}^{d_\text{w}}$, $\phi^{(l)}_{\theta_l}:\mathbb{R}^{d_\text{w}}\to\mathbb{R}^{d_\text{w}}$, $l=2,\dots,L,$ and $\psi_{\theta'}:\mathbb{R}^{d_\text{w}}\to\mathbb{R}^C.$
Here $\mathbf{a}^{(l-1)}$ and $\mathbf{a}^{(l)}$ are the input and output feature vectors of
the $l$-th layer, respectively, and 
each $\phi^{(l)}_{\theta_l}$ is a parameterized function
(e.g., a fully connected layer with nonlinearity). The specific choice of $\phi^{(l)}_{\theta_l}$ and $\psi_{\theta'}$ depends on the INR type (see Supp. \cref{sec:layer_func}).
All parameters $\theta=\{\theta_1,\dots,\theta_L,\theta'\}$ are optimized jointly
forming the baseline single-grade INR in our framework.

\noindent\textbf{Multi-Grade INR Structure and Training.}
The Multi-Grade INR (MG-INR) is an additive stage-wise construction: for an INR with $L$ hidden layers into $L$ sequential \textit{grades}. At grade $l$ we append a new shallow network with one hidden layer $\phi^{(l)}_{\theta_l}$ and one \textit{dedicated} output layer $\psi_{\theta_l'}$ that learns the residual w.r.t the frozen aggregated model from previous grades. This converts the single large non-convex optimization into a sequence of smaller, tractable subproblems. Each grade $l$ learns the \textbf{residual error} of the network up to grade $l-1$, and parameters from previous grades are frozen, acting as fixed adaptive ``basis'' functions \cite{xu2025multi}.

Let $\mathbf{x} \in \Omega \subset \mathbb{R}^2$ be an input coordinate. In grade $l$ ($l=1,\dots,L$), the grade network $g^{(l)}$ computes the residual output:
\begin{equation}
\begin{aligned}
&\mathbf{a}^{(l)} \coloneqq \phi^{(l)}_{\theta_l}(\mathbf{a}^{(l-1)*}), \quad \text{with } \mathbf{a}^{(0)*} = \mathbf{x} \\
&g^{(l)}_{\theta^{(l)}}(\mathbf{a}^{(l-1)*}) \coloneqq \psi_{\theta'_l}(\mathbf{a}^{(l)}).
\end{aligned}
\label{eq:multigrade_step1}
\end{equation}
Here, $\mathbf{a}^{(l-1)*}$ is the features learned in previous grades and serves as input in grade $l$,
and $\theta^{(l)} \coloneqq \{\theta_l, \theta'_l\}$ denotes the parameters specific to grade $l$.

The aggregated network function up to grade $l$, denoted $f^{(l)}$, is the sum of all individual grade outputs:
\begin{equation}
f^{(l)}_{\Theta_{\le l}^*}(\mathbf{x}) = \sum_{j=1}^{l} g^{(j)}_{\theta^{(j)*}}(\mathbf{a}^{(j-1)*}). \label{eq:multigrade_step2}
\end{equation}
where $\Theta_{\le l}^* \coloneqq \{\theta^{(1)*}, \dots, \theta^{(l)*}\}$ represents the set of all optimal (frozen) parameters learned up to grade $l$.

\noindent\textbf{Training.}
Training proceeds sequentially over grades $l=1,\dots,L$. 
At grade $l$, a new learnable function $g^{(l)}_{\theta^{(l)}}(\mathbf{a}^{(l-1)*})$ is added to the previously learned function, yielding the aggregated model
$$
f^{(l)}_{\theta^{(l)}; \ \Theta_{\le l-1}^*}(\mathbf{x})  = g^{(l)}_{\theta^{(l)}}(\mathbf{a}^{(l-1)*}) + f^{(l-1)}_{\Theta_{\le l-1}^*}(\mathbf{x}),
$$
where $\theta^{(l)}$ in  $g^{(l)}_{\theta^{(l)}}$ are the only parameters to be trained in grade $l$. 
Grade $l$ is trained by minimizing the discrepancy between the target $\mathbf{y}$ (e.g., pixel value at $\mathbf{x}$) and the aggregated output:
\begin{equation*}
\theta^{(l)*} = {\rm argmin}_{\theta^{(l)}} \mathcal{L}\left( f^{(l)}_{\theta^{(l)}; \ \Theta_{\le l-1}^*}(\mathbf{x}), \mathbf{y}\right).
\end{equation*}
Crucially, during this step, the parameters from previous grades, $\Theta_{\le l-1}^*$, remain frozen, and only $\theta^{(l)}$ is updated. The final MG-INR is the fully aggregated model, $f^{(L)}_{\Theta_{\le L}^*}(\mathbf{x})$.

The multi-grade strategy yields substantial gains over single-grade approaches. By training sequentially in a coarse-to-fine manner, it achieves higher fidelity with sharper details and more accurate colors on clean images (\cref{fig:sg vs mg fitting}). On restoration tasks, it consistently outperforms single-grade, improving PSNR/SSIM by +1.67 dB/+0.0699 on denoising (\cref{fig:sg vs mg denoising}) and +0.98 dB/+0.0115 on deblurring (\cref{fig:sg vs mg deblurring}). Moreover, it reduces VRAM usage by 30\%. This balance of accuracy, efficiency, and quality makes the multi-grade strategy essential to our framework.

\begin{figure}[htbp]
    \begin{subfigure}{0.24\linewidth}
        \caption{Single-grade}
        \centering
        \includegraphics[width=\linewidth]{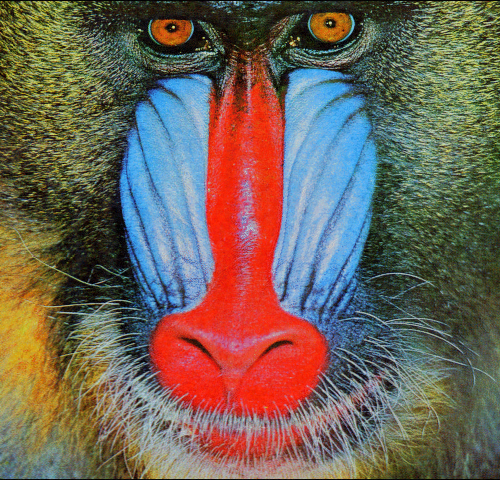}
        \caption*{(29.66 / 0.9236)}
    \end{subfigure}
    \begin{subfigure}{0.24\linewidth}
        \caption{Multi-grade}
        \centering
        \includegraphics[width=\linewidth]{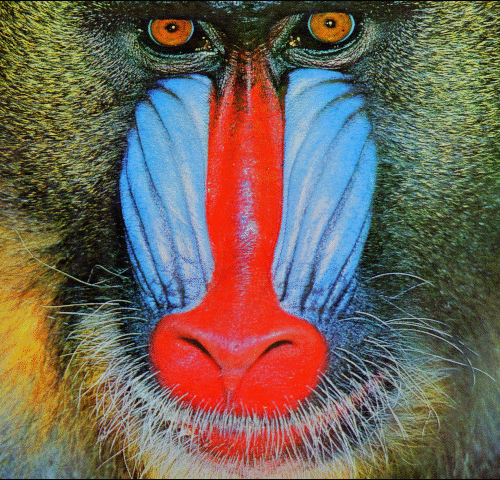}
        \caption*{\textbf{(44.03 / 0.9954)}}
    \end{subfigure}
    \begin{subfigure}{0.231\linewidth}
        \caption{Single-grade}
        \centering
        \includegraphics[width=\linewidth]{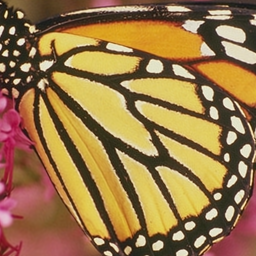}
        \caption*{(37.49 / 0.9837)}
    \end{subfigure}
    \begin{subfigure}{0.231\linewidth}
    \caption{Multi-grade}
        \centering
        \includegraphics[width=\linewidth]{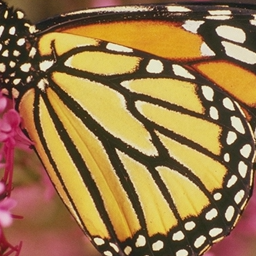}
        \caption*{\textbf{(43.44 / 0.9930)}}
    \end{subfigure}
    \caption{Multi-grade training improves image fitting quality using SIREN.}
    \label{fig:sg vs mg fitting}
\end{figure}

\begin{figure}[htbp]
\begin{subfigure}{0.32\linewidth}
    \caption{Ground truth}
    \centering
    \includegraphics[width=\linewidth]{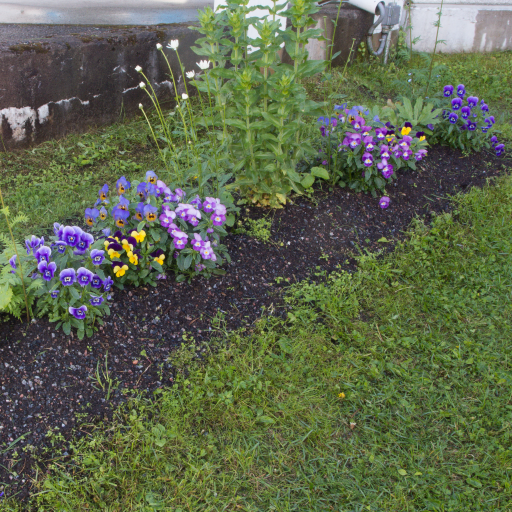}
    \caption*{\quad}
\end{subfigure}
\begin{subfigure}{0.32\linewidth}
    \caption{Single-grade}
    \centering
    \includegraphics[width=\linewidth]{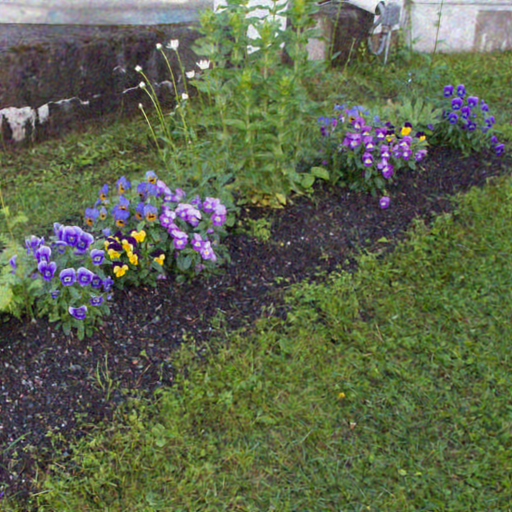}
    \caption*{(26.45 / 0.7844)}
\end{subfigure}
\begin{subfigure}{0.32\linewidth}
    \caption{Multi-grade}
    \centering
    \includegraphics[width=\linewidth]{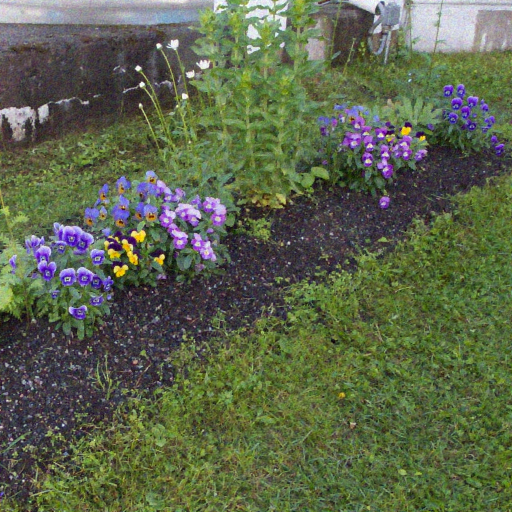}
    \caption*{\textbf{(28.12 / 0.8543)}}
\end{subfigure}
\caption{Multi-grade training improves denoising performance using SIREN with noise level $\sigma_\text{noise} = 15/255$.}
\label{fig:sg vs mg denoising}
\end{figure}

\begin{figure}[htbp]
\begin{subfigure}{0.32\linewidth}
    \caption{Ground truth}
    \centering
    \includegraphics[width=\linewidth]{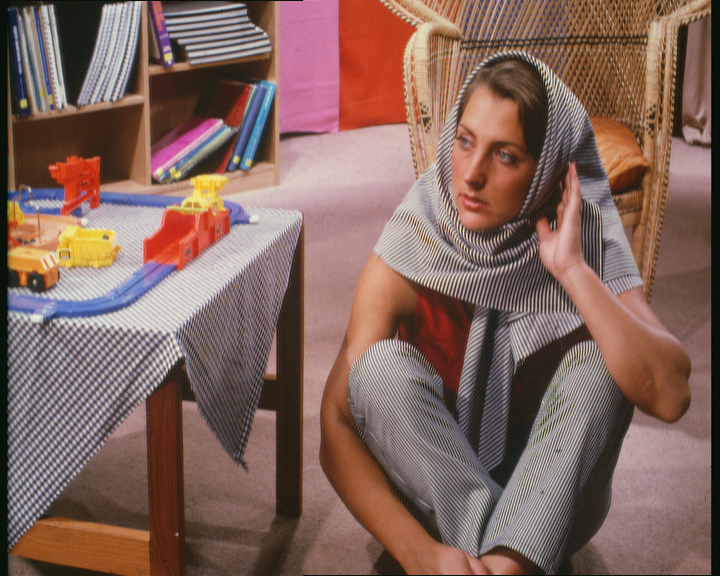}
    \caption*{\quad}
\end{subfigure}
\begin{subfigure}{0.32\linewidth}
    \caption{Single-grade}
    \centering
    \includegraphics[width=\linewidth]{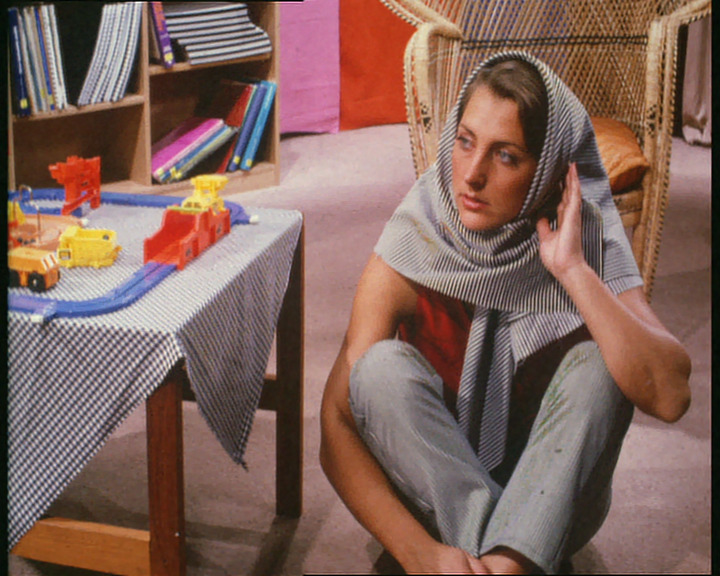}
    \caption*{(27.45 / 0.8443)}
\end{subfigure}
\begin{subfigure}{0.32\linewidth}
    \caption{Multi-grade}
    \centering
    \includegraphics[width=\linewidth]{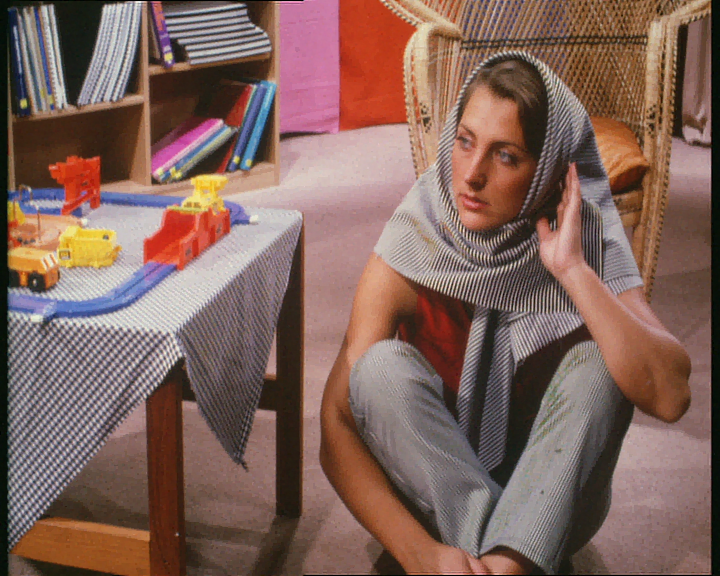}
    \caption*{\textbf{(28.43 / 0.8558)}}
\end{subfigure}
\caption{Multi-grade training improves deblurring performance using SIREN with $\sigma_\text{blur} = 1.0$ and $\sigma_\text{noise} = 3/255$.}
\label{fig:sg vs mg deblurring}
\end{figure}

\noindent\textbf{Advantages.}
MG strategy offers several advantages. By decomposing training into sequential grades, each solving a smaller and more tractable optimization problem, it alleviates gradient-related difficulties. With a single hidden layer and ReLU activation, each grade yields a convex optimization \cite{pilanci2020neural}, making MG strategy effectively a sequence of convex subproblems \cite{fang2025computational}. The incremental learning of high-frequency components mitigates the spectral bias of SG strategy and captures details often overlooked in single-grade training. Moreover, each grade can be  analyzed independently, revealing how features accumulate and residuals diminish, which enhances interpretability. Computationally, MG strategy scales linearly with the number of grades and requires storing only the current grade’s parameters, unlike SG strategy, which updates all layers jointly. This multi-grade strategy increases INR expressivity, enabling the recovery of fine details from sparse or low-resolution inputs. As illustrated in \cref{fig:INR vs bicubic comparison}, INRs yield smoother, more continuous structures than bicubic interpolation. Consequently, we adopt the multi-grade framework throughout this paper.


\begin{figure}[htbp]
    \begin{subfigure}{0.32\linewidth}
        \caption{Ground Truth}
        \centering
        \includegraphics[width=\linewidth]{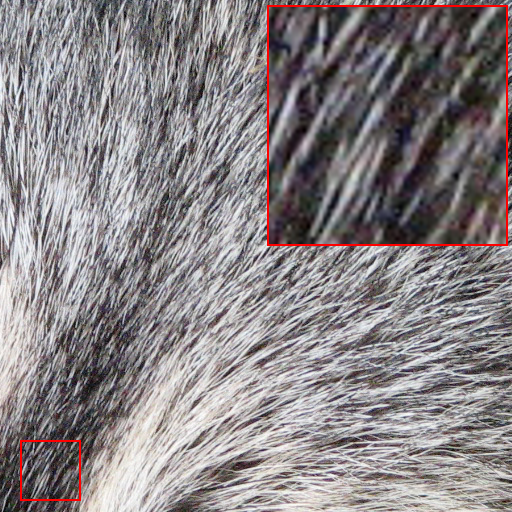}
    \end{subfigure}
    \begin{subfigure}{0.32\linewidth}
        \caption{Bicubic}
        \centering
        \includegraphics[width=\linewidth]{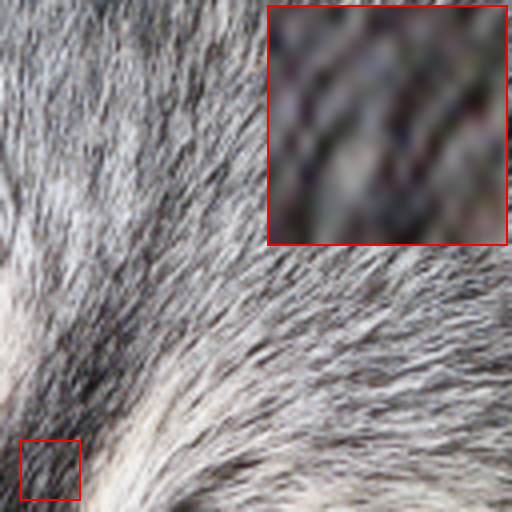}
    \end{subfigure}
    \begin{subfigure}{0.32\linewidth}
        \caption{Multi-grade INR}
        \centering
        \includegraphics[width=\linewidth]{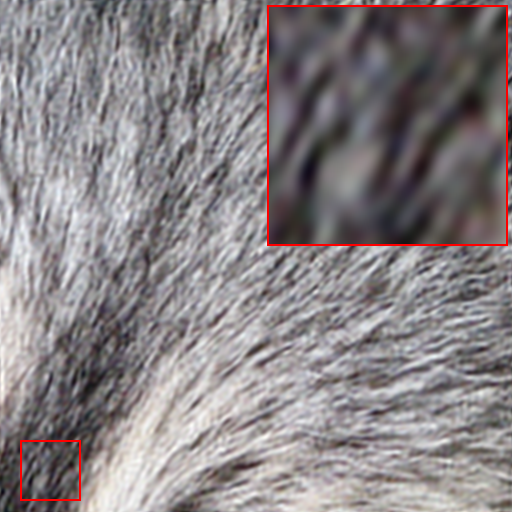} 
    \end{subfigure}
    \caption{Comparison of multi-grade INR and bicubic interpolation on the super-resolution task.}
    \label{fig:INR vs bicubic comparison}
\end{figure}

\subsection{Regularizing INRs} 
We further compare regularization strategies in \cref{fig:reg_vs_no_reg superres}. “Low-res reg” applies regularization to the low-resolution output (removing $\mathcal{S}_s$ and replacing $f_{\theta}(\mathbf{X}_\text{high})$ with $f_{\theta}(\mathbf{X}_\text{low})$ in \cref{eq:original}), whereas “High-res reg” applies it to the target high-resolution output. A small low-res regularization is insufficient to remove artifacts, while a large one causes oversmoothing. In contrast, our “Small high-res reg” (\cref{fig:reg_comp_siren sh}) effectively suppresses artifacts while preserving fine details. This advantage of high-resolution regularization holds across INR backbones, as shown for WIRE in Supp. \cref{sec:WIRE reg}.

These observations confirm that direct INR fitting requires explicit regularization, particularly in denoising, super-resolution, and inpainting tasks—where noise further amplifies instability.
We therefore propose a sparse high-resolution regularization that suppresses spurious high-frequency artifacts and encourages observation-consistent fine details while preserving edge sharpness and piecewise-smooth structures.  High-resolution regularizing achieves better artifact suppression with smaller weights than low-resolution regularization, avoiding oversmoothing.

\begin{figure}[htbp]
    \centering
    \begin{subfigure}{0.32\linewidth}
        \caption{Input ($2\times\downarrow$)}
        \centering
        \includegraphics[width=\linewidth]{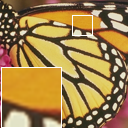}
        \caption*{\quad}
    \end{subfigure}
    \begin{subfigure}{0.32\linewidth}
    \caption{WO/reg}
        \centering
        \includegraphics[width=\linewidth]{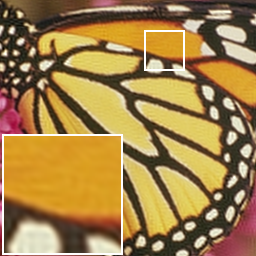}
        \caption*{(24.94 / 0.8816)}
        \label{fig:reg_comp_siren wo}
    \end{subfigure}
    \begin{subfigure}{0.32\linewidth}
        \caption{Small low-res reg}
        \centering
        \includegraphics[width=\linewidth]{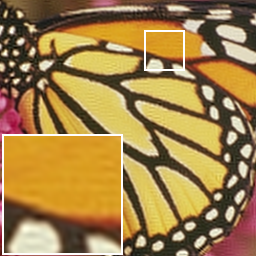}
        \caption*{(24.85 / 0.8722)}
        \label{fig:reg_comp_siren sl}
    \end{subfigure}
    \begin{subfigure}{0.32\linewidth}
        \caption{Large low-res reg}
        \centering
        \includegraphics[width=\linewidth]{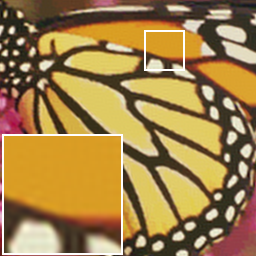}
        \caption*{(24.18 / 0.8207)}
        \label{fig:reg_comp_siren ll}
    \end{subfigure}
    \begin{subfigure}{0.32\linewidth}
           \caption{Small high-res reg}
        \centering
        \includegraphics[width=\linewidth]{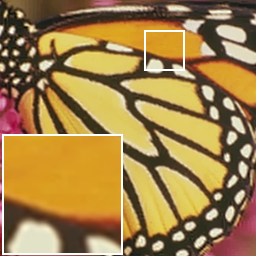}
        \caption*{\textbf{(25.33 / 0.9001)}}
        \label{fig:reg_comp_siren sh}
    \end{subfigure}
    \begin{subfigure}{0.32\linewidth}
            \caption{Ground truth}
        \centering
        \includegraphics[width=\linewidth]{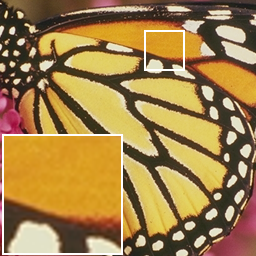}
        \caption*{\quad}
    \end{subfigure}
    \caption{The necessity of high-resolution regularization for super-resolution}
    \label{fig:reg_vs_no_reg superres}
\end{figure}

\begin{figure}[htbp]
    \centering
    \begin{subfigure}{0.32\linewidth}
        \caption{Input ($50\%$ mask)}
        \centering
        \includegraphics[width=\linewidth]{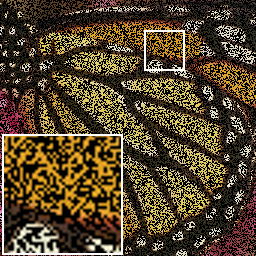}
        \caption*{\quad}
    \end{subfigure}
    \begin{subfigure}{0.32\linewidth}
            \caption{WO/reg}
        \centering
        \includegraphics[width=\linewidth]{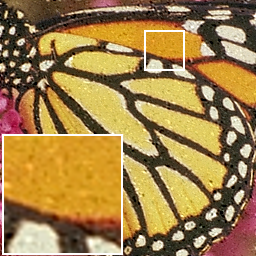}
        \caption*{(24.64 / 0.7754)}
        \label{fig:reg_comp_siren_in wo}
    \end{subfigure}
    \begin{subfigure}{0.32\linewidth}
            \caption{W/reg}
        \centering
        \includegraphics[width=\linewidth]{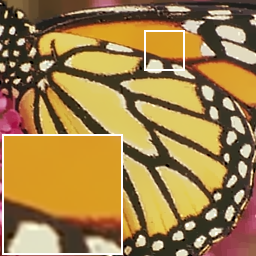}
        \caption*{\textbf{(28.71 / 0.9290)}}
        \label{fig:reg_comp_siren_in with}
    \end{subfigure}
    \caption{The necessity of regularization for inpainting.}
    \label{fig:reg_vs_no_reg inpainting}
\end{figure}


\subsection{The MG-SpaIR Model}\label{section:The MG-SpaIR Model}
We now formalize our model. The objective is to learn INR parameters $\theta$ such that $f_{\theta}$, evaluated on a uniform grid, generates a clean high-resolution image whose degraded version matches the observation $\tilde{\mathbf Y}$, while enforcing sparsity in its gradient field.
To make this optimization tractable, we introduce an auxiliary variable $\mathbf{U}$ representing the image gradients, leading to the following joint problem:
\begin{equation}
\begin{aligned}
    \min_{\mathbf U,\theta} &\underbrace{\mathcal{L}_{\text{fid}} (\tilde{\mathbf Y} - \mathbf M \!\odot\! \mathcal{S}_s \!\circ\! \mathcal{A} \!\circ\! f_{\theta}(\mathbf{X}_\text{high}))}_{\text{Fidelity Term}} \\ 
    + & \underbrace{\frac{\beta}{2}\|\mathbf U-\mathcal{B} \!\circ\! f_{\theta}(\mathbf{X}_\text{high})\|_{\text{F}}^2}_{\text{Surrogate Term}} \quad + \underbrace{\lambda \Phi(\mathbf U)}_{\text{Regularization Term}}.
    \label{eq:original}
\end{aligned}
\end{equation}

Here, $\tilde{\mathbf Y}\in \mathbb{R}^{H_\text{low} \times W_\text{low} \times C}$ is the degraded input, and $s\in \mathbb{N}_+$ denotes the super-resolution factor. The target resolution is $H_\text{high} = s H_\text{low}$ and $W_\text{high} = s W_\text{low}$,  with the coordinate grid $\mathbf{X}_\text{high}\in[-1,1]^2$ uniformly sampled into $H_\text{high}\times W_\text{high}$ points. The arguments of $\mathcal{L}_{\text{fid}}$ are in the low-resolution observed domain, i.e., $\mathbf{A} \in H_\text{low} \times W_\text{low} \times C$. The actual form of the loss $\mathcal{L}_{\text{fid}}$ depends on the assumed noise model. For additive Gaussian noise, we use the mean squared error (MSE):
$\mathcal{L}_{\text{fid}}(\mathbf{A}) := \frac{1}{2} \sum_{h,w,c}\mathbf{A}_{h,w,c}^2$, where $\mathbf{A}$ has dimension $H_\text{low} \times W_\text{low} \times C$.


The second term in \cref{eq:original} introduces a quadratic penalty linking the auxiliary variable $\mathbf{U}\in\mathbb{R}^{H_\text{high}\times W_\text{high}\times 2C}$ to the finite-difference gradient of the high-resolution INR output $f_\theta(\mathbf{X}_\text{high})$. Here, $\|\cdot\|_{\text{F}}$ denotes the Frobenius norm (i.e., the square root of the sum of squared entries).
This surrogate term serves two purposes: it provides a smooth relaxation of the sparsity prior, mitigating staircase artifacts, and it enables efficient optimization via a proximal update of the regularization term.

The operator $\mathcal{B}$ computes per-channel  forward finite differences with Neumann (replicate) padding:
\begin{equation}
\mathcal B \!\circ\! f_\theta(\mathbf{X}_\text{high}) = ( \mathcal{D}_x f_\theta(\mathbf{X}_\text{high}),\, \mathcal{D}_y f_\theta(\mathbf{X}_\text{high}) ), 
\end{equation}
where $\mathcal{D}_x$ and $\mathcal{D}_y$ denote per-channel horizontal and vertical differences, respectively.
Unlike the common isotropic formulation that merges these differences, we adopt an \textit{anisotropic} stacking approach, which, as shown in \cref{fig:ablation_prior}, yields superior restoration quality—likely due to better preservation of directional information. 


The final term, $\lambda \Phi(\mathbf{U})$, imposes a sparse prior on the reconstructed image. Our formulation admits proximal-style updates of $\mathbf{U}$, allowing the prior to be efficiently enforced whenever a closed-form update exists. The $\ell_0$ norm of $\mathbf{U}$, denoted by $\|\mathbf{U}\|_0$, counts the number of non-zero entries in $\mathbf{U}$. In our experiments, we observe that $\ell_0$-norm regularization \cite{fan2001variable, shen2016wavelet, fang2024inexact, xu2023sparse} consistently outperforms alternative choices of $\Phi$, such as $\ell_1$ \cite{chambolle2011first, beck2009fast, micchelli2011proximity} and MCP. Please refer to \cref{fig:effectiveness across sparse reg} for a comparison.

\subsection{MG-SpaIR Algorithm}
\Cref{alg} efficiently solves the highly non-convex problem in \cref{eq:original} via a multi-grade strategy~\cite{xu2025multi} with the proximal alternating minimization~\cite{attouch2010proximal}.

The algorithm solves \cref{eq:original} sequentially, starting from grade $1$. 
In grade $l \ (l=1,\ldots,L)$, the loss function is
\begin{equation}
    \begin{aligned}
\mathcal{L}_{l}(\mathbf{U}, \theta^{(l)})&:=\mathcal{L}_{\text{fid}} \big( \tilde{\mathbf Y} - \mathbf M \!\odot\! \mathcal{S}_s \!\circ\! \mathcal{A} \!\circ\! f^{(l)}_{\theta^{(l)}; \Theta_{\leq l-1}^*}(\mathbf{X}_\text{high})\big)\\
&+ \frac{\beta}{2}\|\mathbf U - \mathcal{B} \!\circ\! f^{(l)}_{\theta^{(l)}; \Theta_{\leq l-1}^*}(\mathbf{X}_\text{high})\|_\mathrm{F}^2
+ \lambda \Phi(\mathbf U).
\end{aligned}
\end{equation}
The algorithm minimizes $\mathcal{L}_{l}(\mathbf{U}, \theta^{(l)})$ using a Gauss-Seidel-type block coordinate descent: at each epoch, $\mathbf U$ is updated first using the current $\theta^{(l)}$, followed by updating $\theta^{(l)}$ with the new $\mathbf{U}$. These two steps are repeated iteratively until convergence.


\noindent\textbf{$\mathbf{U}$-step:}
In this step, $\mathbf{U}$ is updated via a proximal operator:
\begin{equation}
\mathbf{U} \gets \operatorname{prox}_{\frac{\gamma\lambda}{\beta}\Phi}(\gamma\mathcal{B} \!\circ\! f^{(l)}_{\theta^{(l)}; \Theta_{\leq l-1}^*}(\mathbf{X}_\text{high})+(1-\gamma)\mathbf{U}),
\label{eq:prox_update}
\end{equation}
to solve $\min_{\mathbf{U}} \ \mathcal{L}_l(\mathbf {U}, \theta^{(l)})$ with $\theta^{(l)}$ fixed.
The proximal operator $\operatorname{prox}_{\frac{\gamma\lambda}{\beta}\Phi}(\cdot)$ has simple closed forms for $\Phi$ corresponding to $\ell_0$, $\ell_1$, or MCP regularization (Supp.~\cref{sec:prox}).

\noindent\textbf{$\theta$-step:}
With $\mathbf{U}$ fixed,  $\theta^{(l)}$ is updated: 
\begin{equation}
     \theta^{(l)} \gets \operatorname{OptStep}\left(\theta^{(l)}, \frac{\partial \mathcal{L}_l}{\partial \theta^{(l)}}(\mathbf U, \theta^{(l)})\right),
\end{equation}
where $\operatorname{OptStep}(\cdot, \cdot)$ can be a single gradient descent or Adam step.


\begin{algorithm}[htbp!]
\small
\caption{MG-SpaIR Algorithm}
\label{alg}
\begin{algorithmic}
\Require Degraded image $\tilde{\mathbf{Y}}$; max grade $L$; target size $(H,W)$; combination coefficient $\gamma$
\Ensure Restored image $\hat{\mathbf{Y}}$

\State $\mathbf{X}_{\text{high}} \gets$ a uniform grid of size $(H,W)$; \State $\Theta_{\le 0}^* \gets \emptyset$; \quad $\mathbf{a}^{(0)*} \gets \mathbf{X}_{\text{high}}$; \quad $\hat{\mathbf{Y}}^{(0)} \gets \mathbf{0}$;  \quad $\mathbf{U} \gets \mathbf{0}$

\For{$l \gets 1$ \textbf{to} $L$}
    \State Initialize $\theta^{(l)}$
    \Repeat
        \State $\mathbf{a}^{(l)} \gets \phi^{(l)}_{\theta_l}(\mathbf{a}^{(l-1)*})$
        \State $\mathbf{y} \gets \psi^{(l)}_{\theta'_l}(\mathbf{a}^{(l)})$
        \State $\mathbf{Y}^{(l)} \gets \mathbf{y} + \hat{\mathbf{Y}}^{(l-1)}$

        \State $\mathbf{U} \gets \operatorname{prox}_{\frac{\gamma\lambda}{\beta}\,\Phi}\Big(\gamma\,\mathcal{B} \circ \mathbf{Y}^{(l)} + (1-\gamma)\,\mathbf{U}\Big)$

        \State $\theta^{(l)} \gets \operatorname{OptStep}\!\left(\theta^{(l)},\, \frac{\partial \mathcal{L}_l}{\partial \theta^{(l)}}(\mathbf{U},\theta^{(l)})\right)$
    \Until{$\mathcal{L}_l(\cdot,\cdot)$ converges.}

    \State $\theta^{(l)*} \gets \theta^{(l)}$; \quad $\Theta_{\le l}^* \gets \{\theta^{(l)*}\} \cup \Theta_{\le l-1}^*$; \quad $\mathbf{a}^{(l)*} \gets \mathbf{a}^{(l)}$; \quad $\hat{\mathbf{Y}}^{(l)} \gets \mathbf{Y}^{(l)}$
\EndFor

\State $\hat{\mathbf{Y}} \gets \hat{\mathbf{Y}}^{(L)}$
\end{algorithmic}
\end{algorithm}

\section{Convergence Analysis}\label{section: Convergence Analysis}

We analyze the convergence of the MG-SpaIR algorithm.

For a fixed grade $l \in \{1, \ldots, L\}$, let $\mathcal{L}_\text{fid}$ denote the MSE loss, $\Phi$ the $\ell_0$ regularization, and $\operatorname{OptStep}$ a gradient descent update with step size $\eta$.  
Define $f(\theta^{(l)}) \coloneqq f^{(l)}_{\theta^{(l)}; \Theta_{\le l-1}^*}(\mathbf{x})$ and, for brevity, omit the superscripts of $\mathcal{L}_l$ and $\theta^{(l)}$.  
Under these settings, the $k$-th iteration of \cref{alg} is

\begin{equation}\label{eq:proximal_gradient}
\begin{aligned}
&\mathbf{U}^{k+1} \in \operatorname{prox}_{\frac{\gamma\lambda}{\beta}\Phi}\!\bigl(\gamma\mathcal{B}\!\circ\! f(\theta^{k}) + (1-\gamma)\mathbf{U}^{k}\bigr),\\
&\theta^{k+1} = \theta^{k} - \eta\,\frac{\partial \mathcal{L}}{\partial \theta}(\mathbf{U}^{k+1}, \theta^k).
\end{aligned}
\end{equation}
Here $\mathbf{U}$ and $\theta$ (with or without superscripts) denote their flattened vectors of total dimension $d$.  

If the activation of $f$ is twice continuously differentiable, let $\bm{H}_\mathcal{L}(\mathbf{U}, \theta)$ be the Hessian of $\mathcal{L}$ with respect to $\theta$, and define
\begin{equation}\label{eq:alpha_hessian}
\alpha := \sup\{\|\bm{H}_\mathcal{L}(\mathbf{U}, \theta)\|_2 : (\mathbf{U}, \theta) \in \Omega\},
\end{equation}
where $\|\cdot\|_2$ denotes the spectral norm.  

\begin{theo}\label{theorem:convergence:algorithm}
Let $\{(\mathbf{U}^k, \theta^k)\}_{k=0}^\infty$ be generated by \eqref{eq:proximal_gradient} with initial $(\mathbf{U}^0, \theta^0)$.  
Assume:
\begin{enumerate}[label=\textup{(A\arabic*)}]
    \item the activation of $f$ is twice continuously differentiable;
    \item there exists a convex, compact $\Omega \subset \mathbb{R}^d$ such that 
    \[
    \{(\mathbf{U}^k, \theta^k)\}, \ \{(\mathbf{U}^{k+1}, \theta^k)\} \subset \Omega;
    \]
    \item $\alpha$ in \eqref{eq:alpha_hessian} is finite.
\end{enumerate}
If $\eta \in (0, 2/\alpha)$ and $\gamma \in (0,1)$, then:
\begin{enumerate}[label=\textup{(\roman*)}]
    \item $\displaystyle \lim_{k\to\infty}\mathcal{L}(\mathbf{U}^k, \theta^k) = L^*$ for some $L^* \ge 0$;
    \item $\displaystyle \lim_{k\to\infty}\|\mathbf{U}^{k+1} - \mathbf{U}^k\|_2 = 0$ and $\displaystyle \lim_{k\to\infty}\|\theta^{k+1} - \theta^k\|_2 = 0$;
    \item $\displaystyle \lim_{k\to\infty}\mathrm{dist}\!\left(\bm{0}, \tfrac{\partial\mathcal{L}}{\partial\mathbf{U}}(\mathbf{U}^{k+1}, \theta^k)\right) = 0$ and $\displaystyle \lim_{k\to\infty}\tfrac{\partial\mathcal{L}}{\partial\theta}(\mathbf{U}^{k+1}, \theta^k) = 0$.
\end{enumerate}
Here $\tfrac{\partial\mathcal{L}}{\partial\mathbf{U}}$ denotes the subdifferential of $\mathcal{L}$ with respect to $\mathbf{U}$,  
$\tfrac{\partial\mathcal{L}}{\partial\theta}$ its gradient with respect to $\theta$,  
and $\mathrm{dist}(\bm{p}, S) := \inf_{\bm{q}\in S}\|\bm{p}-\bm{q}\|_2$ is the distance from $\bm{p}$ to set $S$.
\end{theo}


The proof of \cref{theorem:convergence:algorithm} is given in Appendix~\ref{Supplement: MG-SpaIR Algorithm}.
\cref{theorem:convergence:algorithm} establishes the convergence of \cref{alg} under mild conditions when $\eta < 2 / \alpha$, where $\eta$ is the learning rate and $\alpha$ the Hessian spectral norm. This result indicates that a smaller Hessian spectral norm permits a larger stable learning rate. Within each grade, the multi-grade strategy trains a shallow network with fewer parameters, resulting in a smaller Hessian spectral norm than its single-grade counterpart and thus allowing a higher learning rate. Moreover, the assumption that $f$ is twice continuously differentiable is satisfied by the sinusoidal and Gabor-type activations used in our backbones (SIREN/WIRE), aligning the theory with our experimental setting. 

To demonstrate this effect, we perform an image denoising experiment (\cref{fig: Compare SGDL and MGDL main text}).
Both the single-grade and multi-grade networks, using a SIREN~\cite{sitzmann2020implicit} backbone, are applied to the image denoising problem described in \cref{section:The MG-SpaIR Model}. The input image is corrupted by additive Gaussian noise with zero mean and a standard deviation of $10/255$.
For the single-grade network, four hidden layers are used. For the multi-grade network, three grades are employed: the first and second grades each contain one hidden layer, while the third grade contains two hidden layers. The learning rate is varied from $1\times10^{-6}$ to $5\times10^{-5}$. The single-grade model achieves its highest PSNR at $\eta = 5\times10^{-6}$, whereas the multi-grade model attains its peak PSNR at a larger learning rate, $\eta = 2\times10^{-5}$. 
This observation indicates that the multi-grade model exhibits greater training stability under larger learning rates.

As shown in \cref{fig: Compare SGDL and MGDL main text (a)}, the multi-grade strategy consistently yields a smaller Hessian spectral norm, while \cref{fig: Compare SGDL and MGDL main text (b)} shows that, under the same learning rate, the single-grade strategy diverges whereas the multi-grade one remains stable and continues to improve. These results confirm that the multi-grade strategy enables a larger stable learning rate.



\begin{figure}[htbp]
\begin{subfigure}{0.32\linewidth}
    \centering
    \includegraphics[width=\linewidth]{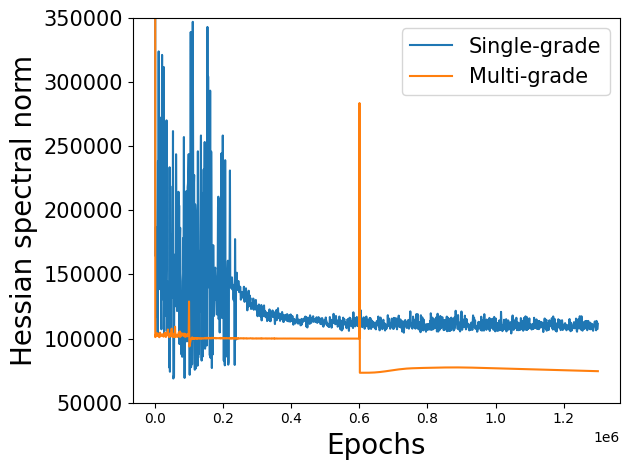}
    \caption{Hessian spectral norm}
    \label{fig: Compare SGDL and MGDL main text (a)}
\end{subfigure}
\begin{subfigure}{0.32\linewidth}
    \centering
    \includegraphics[width=\linewidth]{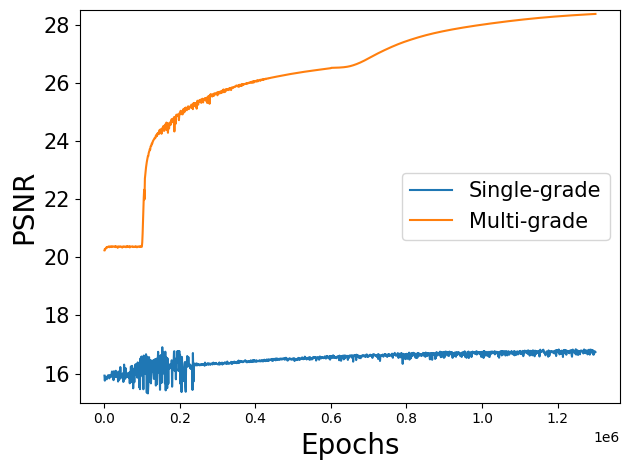}
    \caption{PSNR (in dB)}
    \label{fig: Compare SGDL and MGDL main text (b)}
\end{subfigure}
\begin{subfigure}{0.32\linewidth}
    \centering
    \includegraphics[width=\linewidth]{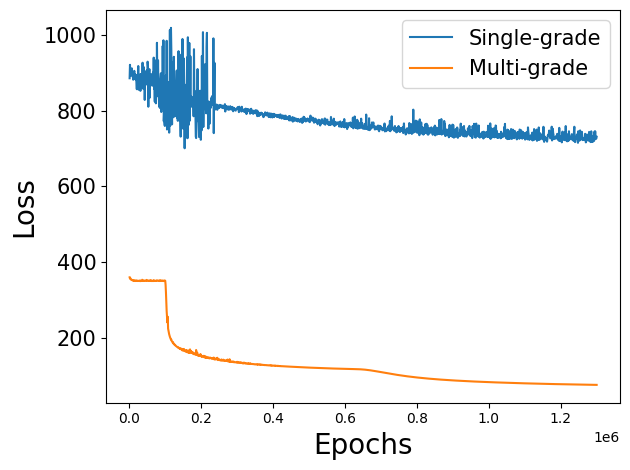}
    \caption{Loss value}
    \label{fig: Compare SGDL and MGDL main text (c)}
\end{subfigure}    
\caption{Evolution of Hessian spectral norm, PSNR, and loss during optimization for single-grade vs. multi-grade strategies ($\eta = 2\times10^{-5}$).}
\label{fig: Compare SGDL and MGDL main text}
\end{figure}

\section{Experiments}\label{sec:experiments}
We conduct experiments to evaluate MG-SpaIR by: (1) demonstrating its effectiveness on mixed-degradation tasks compared with neural network--based and traditional methods, (2) comparing with data-driven methods to illustrate the trade-off between reconstruction fidelity and hallucinated textures, and (3) validating key design choices through ablation studies. Additional ablations are provided in Supp.~\cref{sec:supp_exp}.

\noindent\textbf{Experimental Setup.}
Center-cropped images from Set14 and Flickr2K are used as ground truths, degraded according to the mixed-degradation model in \cref{eq:degradation}. The degradation combines several corruptions: $\mathcal{A}$, a Gaussian blur with kernel standard deviation $\sigma_\text{blur}=0.5$ or $1$, and kernel size $2\lceil 3\sigma_\text{blur}\rceil + 1$; $\mathcal{S}_s$, a $2\times$ bicubic downsampling; $\mathcal{N}$, additive Gaussian noise with $\sigma_{\text{noise}}=5/255$ or $10/255$; and $\mathbf{M}$, a random pixel mask with dropout probability $p_{\text{missing}}=0$ or $0.2$.
The objective uses MSE as $\mathcal{L}_{\text{fid}}$ and an $\ell_0$ norm as $\Phi$. PSNR and SSIM values are reported below each image. The framework is implemented in PyTorch and executed on an RTX~4090 GPU with 24 GB of VRAM.


\subsection{Comparison with Regularized Methods}
We evaluate MG-SpaIR on a challenging mixed-degradation task defined in \cref{eq:degradation}, comparing it with two representative baselines: (1) a traditional restoration pipeline built from classical algorithms, and (2) Deep Image Prior (DIP)~\cite{ulyanov2018deep}, which also exploits neural network structures for image restoration. To ensure fairness, DIP is optimized using a unified end-to-end loss rather than a sequential pipeline, which can accumulate errors.

The traditional baseline follows a sequential restoration process that inverts the degradation model. It consists of four stages: inpainting via the fast marching method~\cite{telea2004image}, denoising with BM3D~\cite{dabov2007image}, bicubic super-resolution~\cite{keys2003cubic}, and Wiener deblurring~\cite{10.7551/mitpress/2946.001.0001}. All degradation parameters are provided to their respective algorithms for a fair comparison. Further implementation details of this pipeline are given in the supplementary material.


\begin{figure}[htbp]
    \centering
    \begin{subfigure}{0.32\linewidth}
        \caption{Input}
        \centering
        \includegraphics[width=\linewidth]{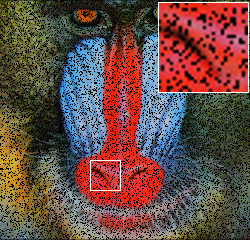}
        \caption*{\quad}
    \end{subfigure}
    \begin{subfigure}{0.32\linewidth}
        \caption{Deep Image Prior}
        \centering
        \includegraphics[width=\linewidth]{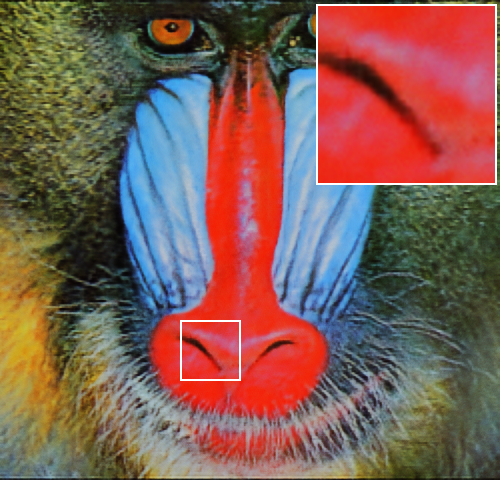}
        \caption*{(21.18 / 0.5321)}
    \end{subfigure}
    \begin{subfigure}{0.32\linewidth}
        \caption{Pipeline}
        \centering
        \includegraphics[width=\linewidth]{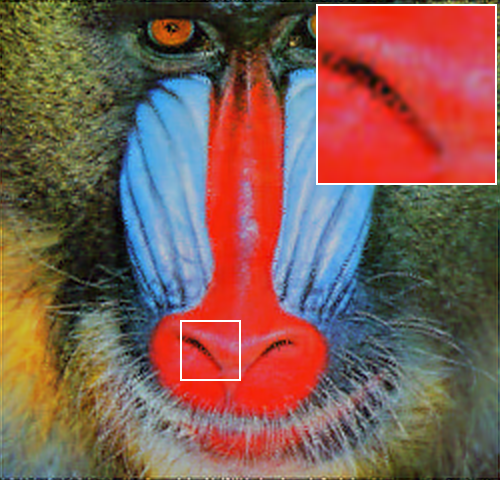}
        \caption*{(21.19 / 0.5364)}
    \end{subfigure}
    \begin{subfigure}{0.32\linewidth}
        \caption{MG-SpaIR-WIRE}
        \centering
        \includegraphics[width=\linewidth]{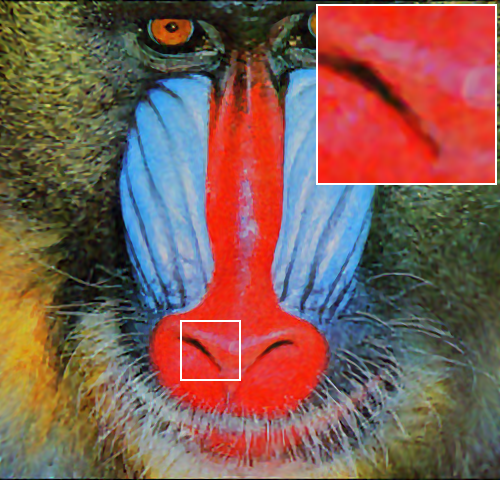}
        \caption*{(21.52 / \textbf{0.5634})}
    \end{subfigure}
    \begin{subfigure}{0.32\linewidth}
        \caption{MG-SpaIR-SIREN}
        \centering
        \includegraphics[width=\linewidth]{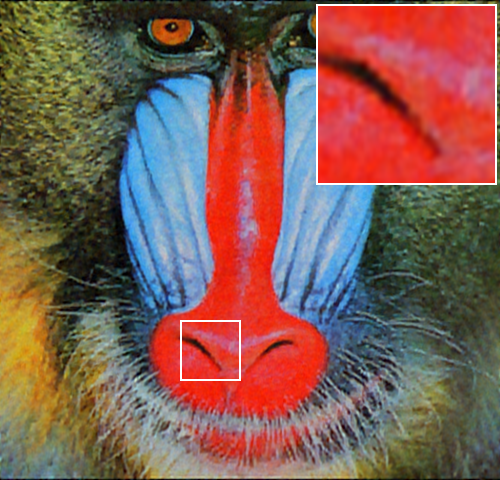}
        \caption*{(\textbf{21.55} / 0.5616)}
    \end{subfigure}
    \begin{subfigure}{0.32\linewidth}
        \caption{Ground Truth}
        \centering
        \includegraphics[width=\linewidth]{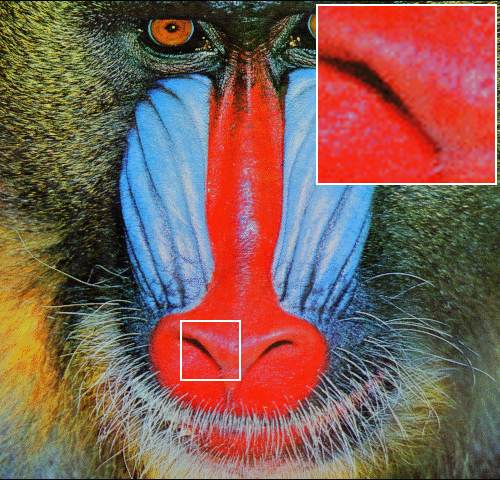}
        \caption*{\quad}
    \end{subfigure}
    \caption{Image restoration of the Baboon image, showing finer detail recovery in irregular regions. Settings: $\sigma_\text{blur}=1$, $2\times$ downsampling, $\sigma_\text{noise}=5$, $p_\text{missing}=0.2$.}
    \label{fig:baboon_comparison}
\end{figure}

\begin{figure}[htbp]
    \centering
    \begin{subfigure}{0.32\linewidth}
        \caption{Input}
        \centering
        \includegraphics[width=\linewidth]{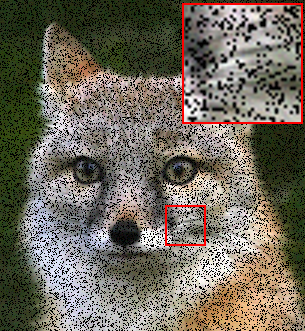}
        \caption*{\quad}
    \end{subfigure}
    \begin{subfigure}{0.32\linewidth}
        \caption{Deep Image Prior}
        \centering
        \includegraphics[width=\linewidth]{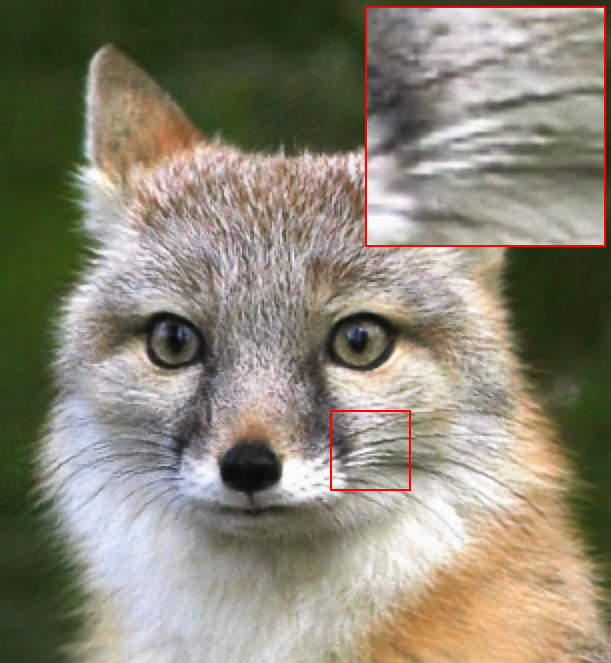} 
        \caption*{(26.15 / 0.7497)}
    \end{subfigure}
    \begin{subfigure}{0.32\linewidth}
        \caption{Pipeline}
        \centering
        \includegraphics[width=\linewidth]{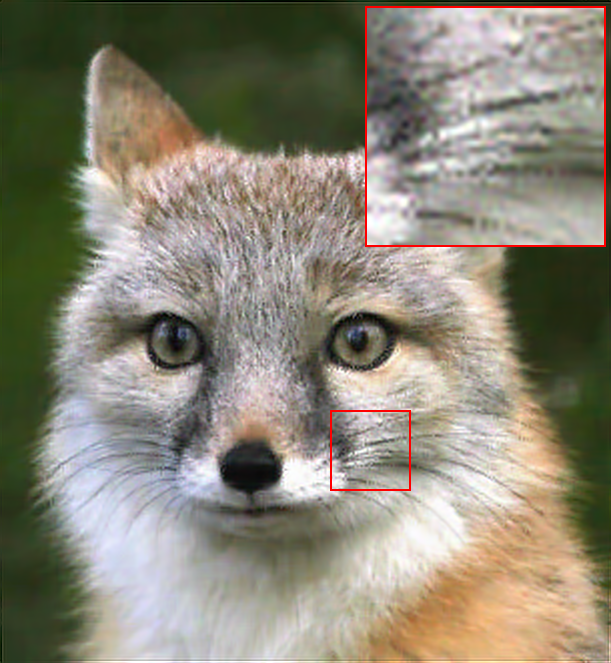}
        \caption*{(26.14 / 0.7618)}
    \end{subfigure}
    \begin{subfigure}{0.32\linewidth}
        \caption{MG-SpaIR-WIRE}
        \centering
        \includegraphics[width=\linewidth]{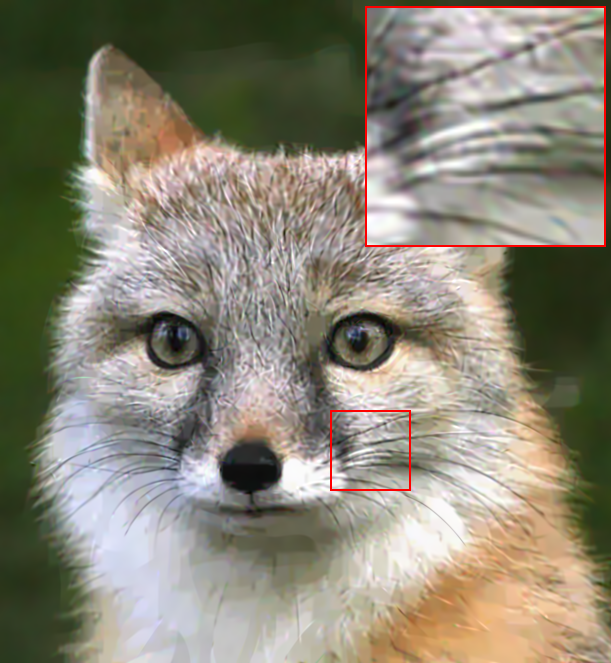}
        \caption*{(26.82 / 0.7803)}
    \end{subfigure}
    \begin{subfigure}{0.32\linewidth}
        \caption{MG-SpaIR-SIREN}
        \centering
        \includegraphics[width=\linewidth]{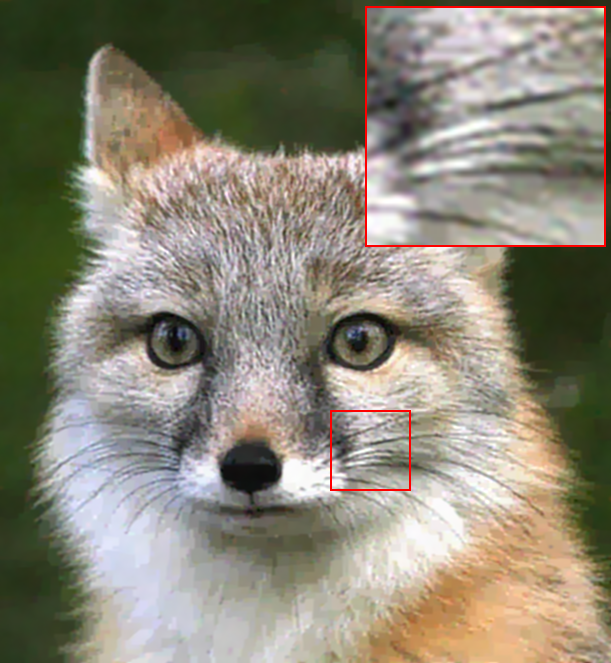}
        \caption*{\textbf{(26.87 / 0.7823)}}
    \end{subfigure}
    \begin{subfigure}{0.32\linewidth}
        \caption{Ground Truth}
        \centering
        \includegraphics[width=\linewidth]{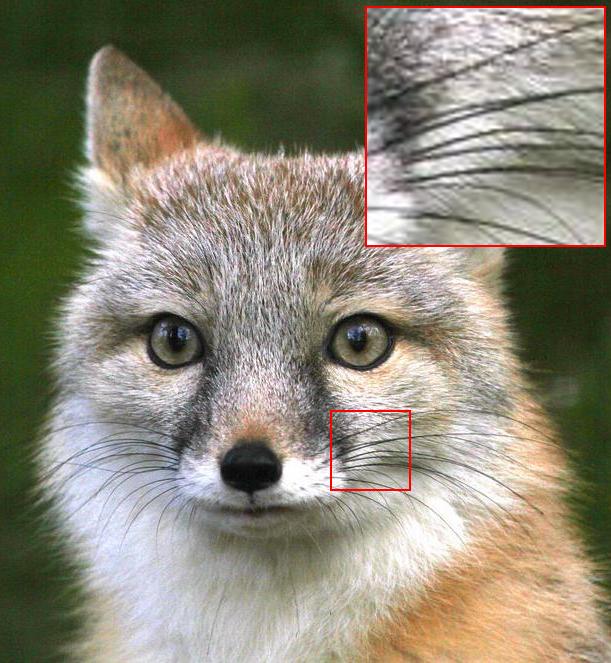}
        \caption*{\quad}
    \end{subfigure}
    \caption{Image restoration of the Fox image, showing smoother and less pixelated line reconstruction. Settings: $\sigma_\text{blur}=1$, $2\times$ downsampling, $\sigma_\text{noise}=5$, $p_\text{missing}=0.2$.}
    \label{fig:fox_comparison}
\end{figure}


\begin{figure}[htbp]
    \centering
    \begin{subfigure}{0.49\linewidth}
        \caption{Input}
        \centering
        \includegraphics[width=\linewidth]{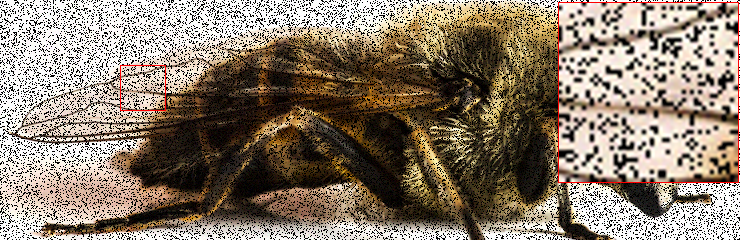}
        \caption*{\quad}
    \end{subfigure}
    \begin{subfigure}{0.49\linewidth}
        \caption{Deep Image Prior}
        \centering
        \includegraphics[width=\linewidth]{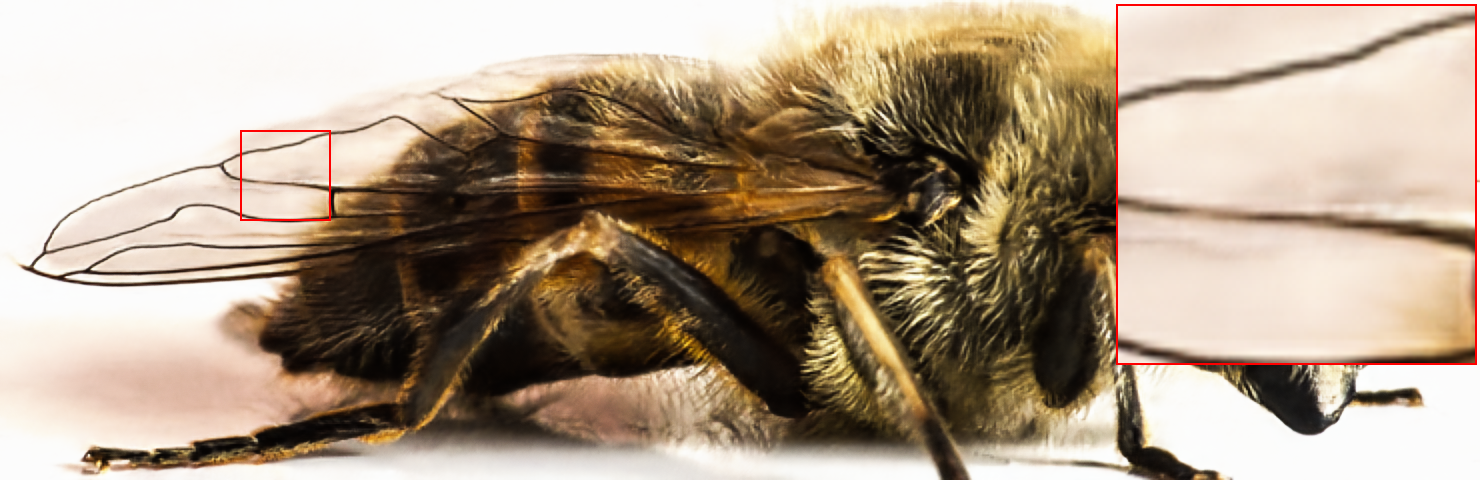}
        \caption*{(25.7 / 0.7924)}
    \end{subfigure}
    \begin{subfigure}{0.49\linewidth}
        \caption{Pipeline}
        \centering
        \includegraphics[width=\linewidth]{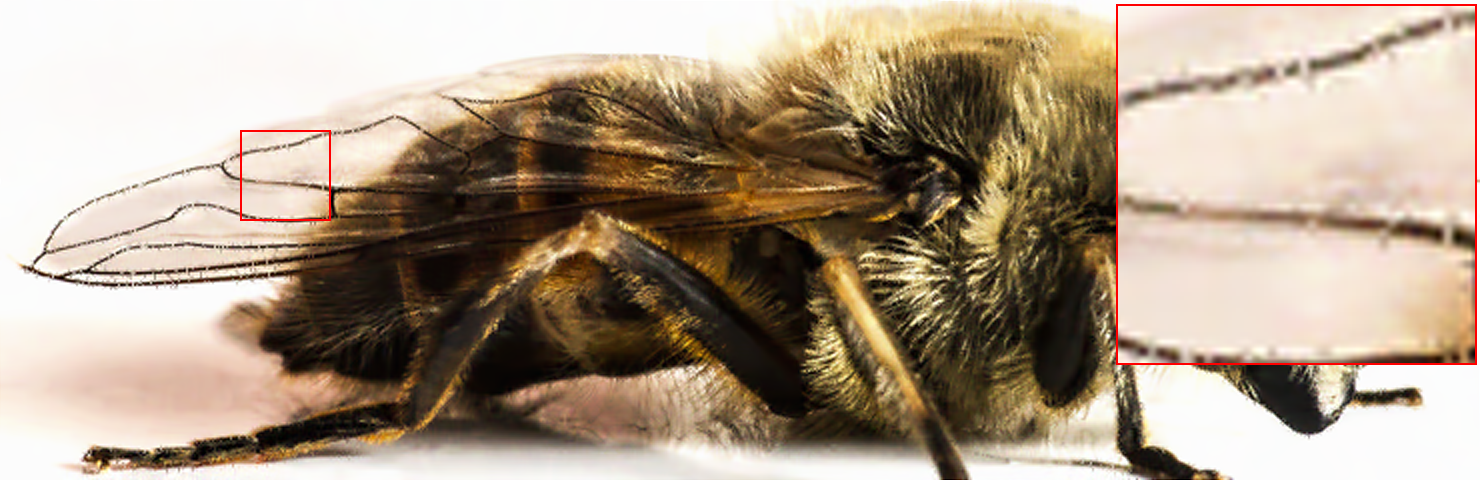}
        \caption*{(25.30 / 0.8021)}
    \end{subfigure}
    \begin{subfigure}{0.49\linewidth}
        \caption{MG-SpaIR-WIRE}
        \centering
        \includegraphics[width=\linewidth]{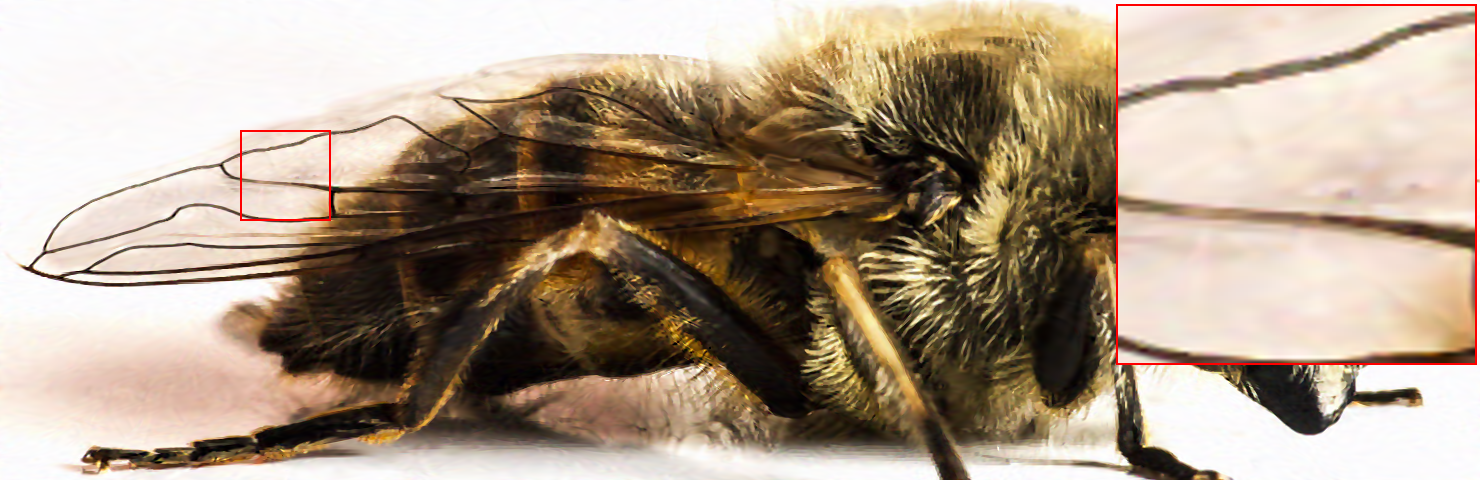}
        \caption*{(25.93 / \textbf{0.8070})}
    \end{subfigure}
    \begin{subfigure}{0.49\linewidth}
        \caption{MG-SpaIR-SIREN}
        \centering
        \includegraphics[width=\linewidth]{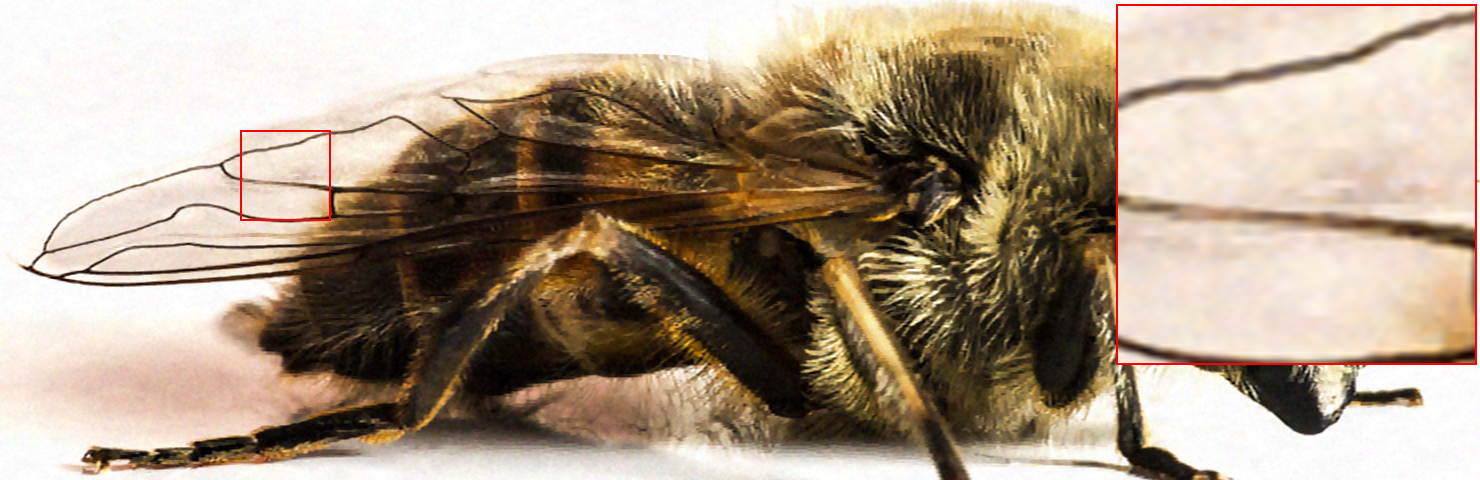}
        \caption*{(\textbf{25.97} / 0.8016)}
    \end{subfigure}
    \begin{subfigure}{0.49\linewidth}
        \caption{Ground Truth}
        \centering
        \includegraphics[width=\linewidth]{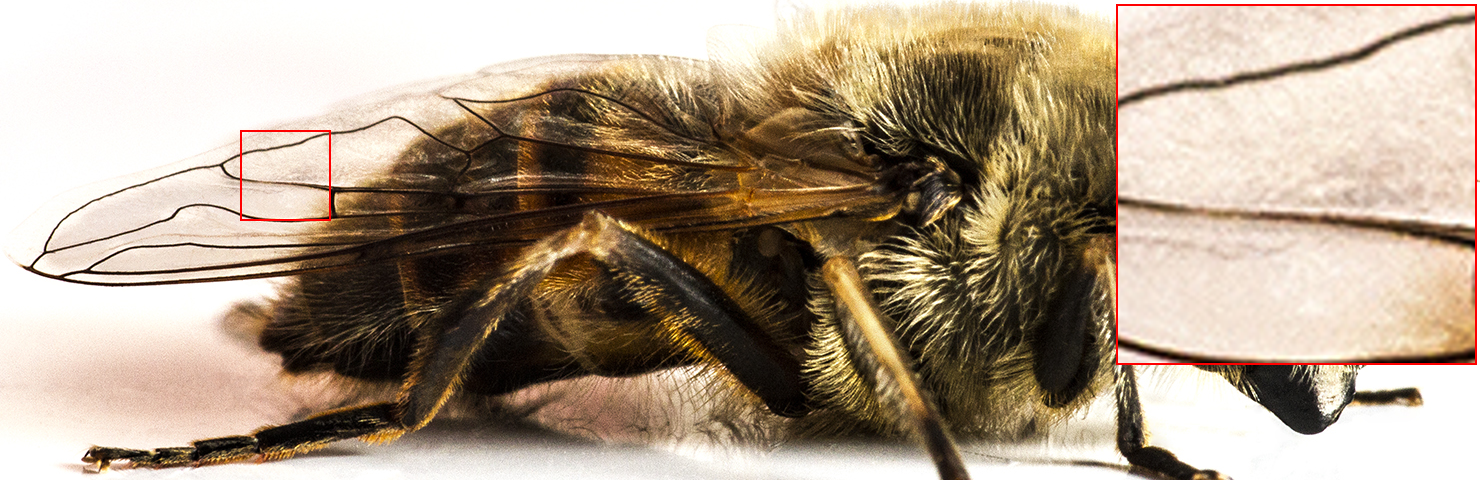}
        \caption*{\quad}
    \end{subfigure}
    \caption{Image restoration of the Bee image, highlighting recovery of fine textures. Settings: $\sigma_\text{blur}=0.5$, $2\times$ downsampling, $\sigma_\text{noise}=5$, $p_\text{missing}=0.2$.}
    \label{fig:bee_comparison}
\end{figure}

\noindent\textbf{Implementation.}
MG-SpaIR can be instantiated with different implicit network backbones. Specifically, \textbf{MG-SpaIR-WIRE} corresponds to the case where $\phi_{\theta_l}^{(l)}$ and $\psi_{\theta_l'}$ follow \cref{WIRE-hidden-layer,WIRE-output-layer}, respectively, while \textbf{MG-SpaIR-SIREN} uses the definitions in  \cref{SIREN-hidden-layer,SIREN-output-layer}, respectively. Both models are trained under the sparse-proximal regularization framework introduced in \cref{sec:model}.


Our methods, MG-SpaIR-WIRE and MG-SpaIR-SIREN, effectively reconstruct fine details (zoom-ins in \cref{fig:baboon_comparison}) and smooth structures (zoom-ins in \cref{fig:fox_comparison,fig:bee_comparison}). DIP fails to recover irregular textures (\cref{fig:baboon_comparison,fig:bee_comparison}) and produces pixelated lines (\cref{fig:fox_comparison}). The traditional pipeline oversmooths fine details (\cref{fig:baboon_comparison}) and fails to preserve structural lines (\cref{fig:bee_comparison}). The combined ability of our methods to preserve structure and recover detail results in the highest average PSNR and SSIM among all compared approaches.

\subsection{Comparison with Data-driven Methods}
\label{sec:compar_data-driven}
We include SwinIR~\cite{liang2021swinir}, a representative data-driven transformer baseline, for comparison under the same composite degradation setting (blur, down-sampling, and noise) without missing pixels, as SwinIR does not support inpainting. Results are reported for two pretrained variants: SwinIR-GAN (perceptual) and SwinIR-PSNR (distortion).

\begin{figure}[htbp!]
    \centering
    \begin{subfigure}{0.32\linewidth}
        \caption{Input}
        \centering
        \includegraphics[width=\linewidth]{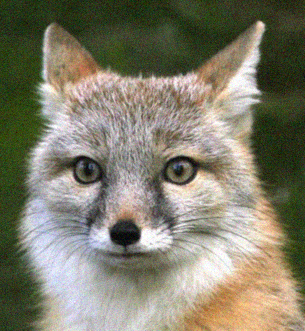}
        \caption*{\quad}
        \label{fig:data_driven input}
    \end{subfigure}
    \begin{subfigure}{0.32\linewidth}
        \caption{SwinIR-GAN}
        \centering
        \includegraphics[width=\linewidth]{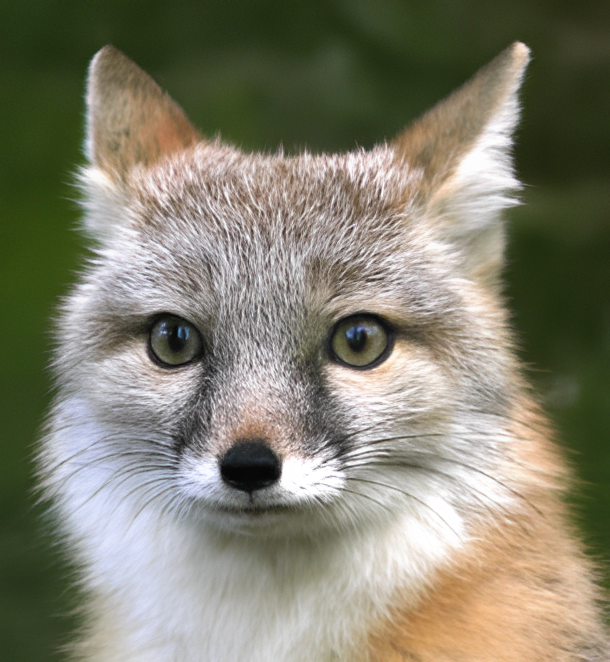}
        \caption*{(24.83 / 0.7107)}
        \label{fig:data_driven GAN}
    \end{subfigure}
    \begin{subfigure}{0.32\linewidth}
    \caption{SwinIR-PSNR}
        \centering
        \includegraphics[width=\linewidth]{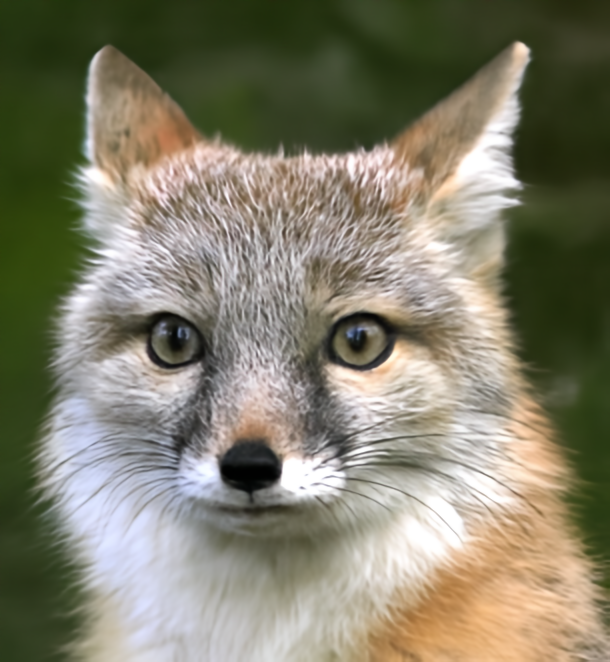}
        \caption*{(25.91 / \textbf{0.7426})}
        \label{fig:data_driven PSNR}
    \end{subfigure}
    \begin{subfigure}{0.32\linewidth}
        \caption{MG-SpaIR-WIRE}
        \centering
        \includegraphics[width=\linewidth]{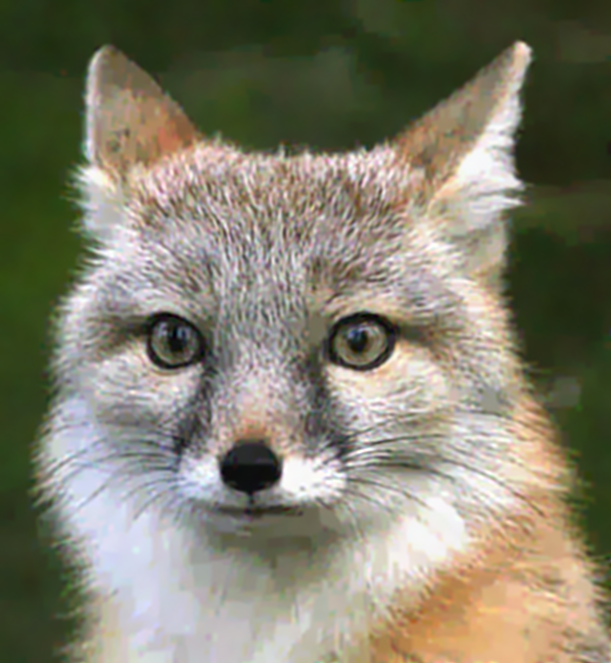}
        \caption*{(25.91 / 0.7301)}
        \label{fig:data_driven WIRE}
    \end{subfigure}
    \begin{subfigure}{0.32\linewidth}
        \caption{MG-SpaIR-SIREN}
        \centering
        \includegraphics[width=\linewidth]{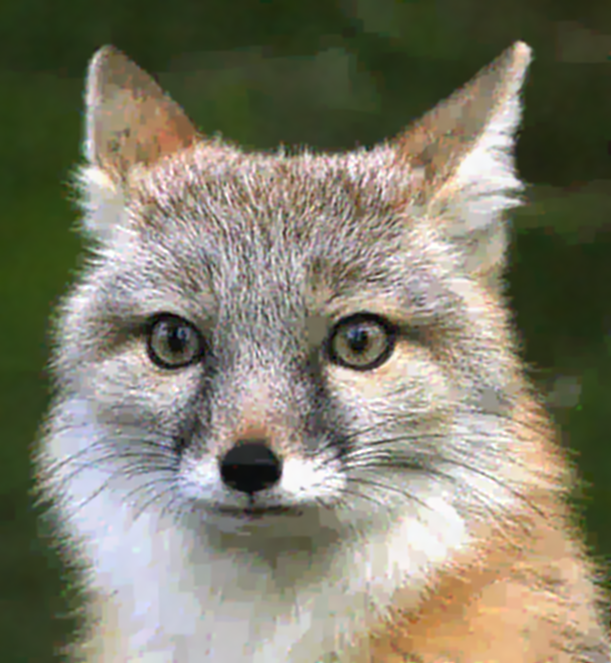}
        \caption*{(\textbf{26.02} / 0.7364)}
        \label{fig:data_driven SIREN}
    \end{subfigure}
    \begin{subfigure}{0.32\linewidth}
        \caption{Ground Truth}
        \centering
        \includegraphics[width=\linewidth]{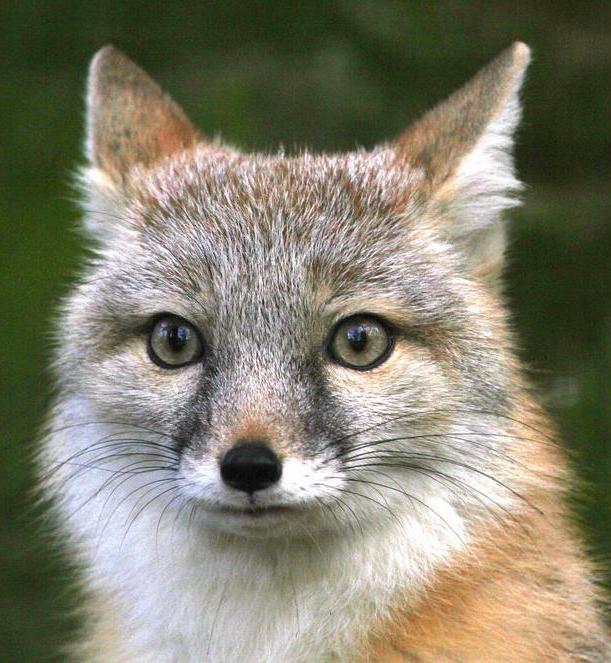}
        \caption*{\quad}
        \label{fig:data_driven GT}
    \end{subfigure}
    \caption{$\sigma_\text{blur}=1$, $2\times$ downsampled, $\sigma_\text{noise}=10$, $p_\text{missing}=0$ (No missing pixels).}
    \label{fig:main_comparison_fox}
\end{figure}

\cref{fig:main_comparison_fox} presents a qualitative comparison on a representative test example. Our MG-SpaIR-based results (\cref{fig:data_driven WIRE,fig:data_driven SIREN}) achieve slightly higher PSNR than the SwinIR baselines, though they preserve mild high-frequency noise and appear less smooth. The SwinIR variants (\cref{fig:data_driven GAN,fig:data_driven PSNR}) produce visually cleaner and sharper outputs; however, part of this sharpness comes from hallucinated high-frequency textures learned from external datasets. These hallucinated details increase perceived sharpness but do not correspond to the true underlying ground-truth structure, as they stem from learned dataset priors rather than information actually present in the ground truth. Consequently, they penalize PSNR and SSIM, explaining why data-driven methods’ visually sharper outputs do not necessarily achieve higher distortion metrics.

Unlike data-driven models such as SwinIR, which may hallucinate textures due to reliance on external training data, our method is guided solely by implicit priors from the INR backbone and explicit handcrafted sparsity priors applied in the image domain.
These two complementary priors jointly steer the optimization toward faithful structures rather than invented details.

\subsection{Ablation Studies}
\textbf{Necessity of Anisotropic $\ell_0$.}
To evaluate our key design choices, we perform an ablation study on the regularization prior. The visual and quantitative results in \cref{fig:ablation_prior} support two main conclusions.

\begin{figure}[htbp]
    \centering
    \begin{subfigure}{0.24\linewidth}
        \caption{Input}
        \centering
        \includegraphics[width=\linewidth]{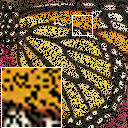}
        \caption*{\quad}
        \label{fig:degraded}
    \end{subfigure}
    \begin{subfigure}{0.24\linewidth}
        \caption{W/O Reg}
        \centering
        \includegraphics[width=\linewidth]{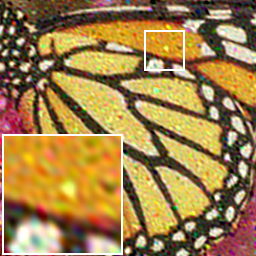}
        \caption*{(20.23 / 0.5842)}
        \label{fig:wo_reg}
    \end{subfigure}
    \begin{subfigure}{0.24\linewidth}
        \caption{W/Iso-$\ell_0$}
        \centering
        \includegraphics[width=\linewidth]{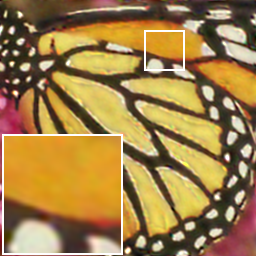}
        \caption*{(22.55 / 0.7495)}
        \label{fig:iso_l0}
    \end{subfigure}
    \begin{subfigure}{0.24\linewidth}
        \caption{W/Aniso-$\ell_0$}
        \centering
        \includegraphics[width=\linewidth]{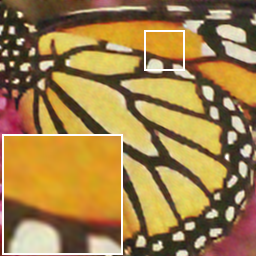}
        \caption*{\textbf{(22.82 / 0.7788)}}
        \label{fig:aniso_l0}
    \end{subfigure}
    \caption{Ablation study on the regularization prior. From left to right: degraded input, reconstruction without an explicit prior, with isotropic $\ell_0$ regularization, and with our proposed anisotropic $\ell_0$ regularization (Ours). Results demonstrate the need for an explicit prior and the superiority of the anisotropic formulation.}
    \label{fig:ablation_prior}
\end{figure}

First, the results confirm the necessity of an explicit prior. As shown in \cref{fig:wo_reg}, relying solely on the INR's implicit regularization is insufficient, producing reconstructions with significant residual noise and low PSNR.

Second, the formulation of the $\ell_0$ prior is critical. \Cref{fig:aniso_l0} demonstrates that the anisotropic version achieves superior restoration quality with the highest PSNR and SSIM, thanks to its ability to preserve directional information. In contrast, the isotropic $\ell_0$ norm (\cref{fig:iso_l0}), applied to the gradient magnitude by combining horizontal and vertical derivatives, can lose orientation cues that guide INR optimization. Our anisotropic formulation decouples these components, allowing independent regularization and improved reconstruction.

\subsection{Effectiveness across Different Sparsity Regularizers} 

To demonstrate the versatility of our proposed method, we apply it to several distinct sparsity regularization including $\ell_0, \ell_1$, and MCP. As shown in \cref{fig:effectiveness across sparse reg}, our method performs consistently across different sparsity regularizers, with $\ell_0$ achieving the best performance.

\begin{figure}[htbp]
\begin{subfigure}{0.32\linewidth}
    \caption{$\ell_0$}
    \centering
    \includegraphics[width=\linewidth]{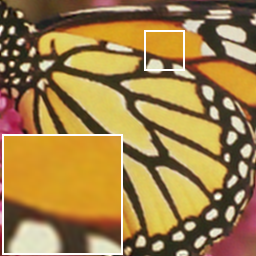}
    \caption*{\textbf{(24.65 / 0.8598)}}
\end{subfigure}
\begin{subfigure}{0.32\linewidth}
    \caption{$\ell_1$}
    \centering
    \includegraphics[width=\linewidth]{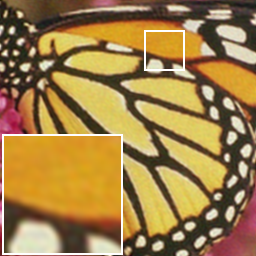}
    \caption*{(24.32 / 0.8457)}
\end{subfigure}
\begin{subfigure}{0.32\linewidth}
    \caption{MCP}
    \centering
    \includegraphics[width=\linewidth]{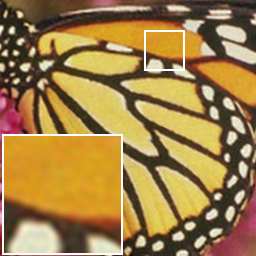}
    \caption*{(24.61 / 0.8513)}
\end{subfigure}
\caption{Effectiveness of our method with different sparse regularizations.}
\label{fig:effectiveness across sparse reg}
\end{figure}

\subsection{Effectiveness across Different INR Backbones.} To demonstrate the versatility of our proposed regularization, we apply it to several distinct INR backbones. As shown in \cref{fig:effectiveness across INR}, our method consistently improves the performance of MLP-based backbones such as SIREN, Gabor-based backbones such as WIRE, and hash-grid-based backbones such as Instant-NGP (``INGP'') \cite{muller2022instant}. This confirms that our proposed prior is a model-agnostic module beneficial for a wide range of INR frameworks. Although the INR itself has implicit regularization \cite{saragadam2023wire}, this experiment shows that adding an explicit regularization for image restoration is helpful and necessary.

\begin{figure}[htbp]
\begin{subfigure}{0.24\linewidth}
    \caption{\tiny{WIRE}}
    \centering
    \includegraphics[width=\linewidth]{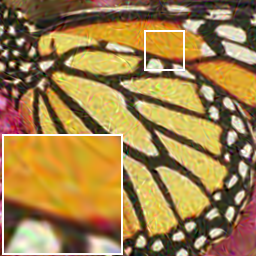}
    \caption*{(21.90 / 0.6725)}
\end{subfigure}
\begin{subfigure}{0.24\linewidth}
    \caption{\tiny{MG-SpaIR-WIRE}}
    \centering
    \includegraphics[width=\linewidth]{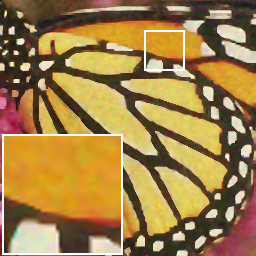}
    \caption*{\textbf{(22.07 / 0.7705)}}
\end{subfigure}
\begin{subfigure}{0.24\linewidth}
    \caption{\tiny{INGP}}
    \centering
    \includegraphics[width=\linewidth]{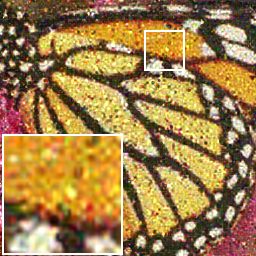}
    \caption*{(17.80 / 0.4736)}
\end{subfigure}
\begin{subfigure}{0.24\linewidth}
    \caption{\tiny{MG-SpaIR-INGP}}
    \centering
    \includegraphics[width=\linewidth]{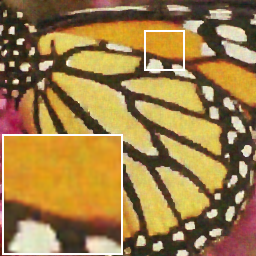}
    \caption*{\textbf{(21.67 / 0.7346)}}
\end{subfigure}
\caption{Effectiveness of our regularization across different INR backbones.}
\label{fig:effectiveness across INR}
\end{figure}

\noindent\textbf{Effectiveness of the Multi-Grade Strategy.} 
To evaluate the multi-grade strategy, we compare it using SIREN against two counterparts: a plain single-grade SIREN and an architectural variant with residual connections (“Residual”)~\cite{he2016deep}. The visual results in \cref{fig:MGDL ablation} show that the multi-grade strategy achieves higher PSNR and SSIM than the single-grade approach. Notably, the residual connection variant suffers from over-smoothing and loss of fine textures, resulting in worse scores than the plain single-grade model.

Regarding memory efficiency, the multi-grade strategy consumes only 5608 MB of VRAM, a reduction of approximately 30\% compared to the 8024 MB required by both the single-grade and residual-connection counterparts.
By achieving both higher reconstruction quality and lower VRAM usage, we conclude that the multi-grade strategy—rather than simple residual connections—is the proper way to implement the residual principle in INRs.


\begin{figure}[htbp]
\centering
\begin{subfigure}{0.32\linewidth}
    \caption{Single-grade}
    \centering
    \includegraphics[width=\linewidth]{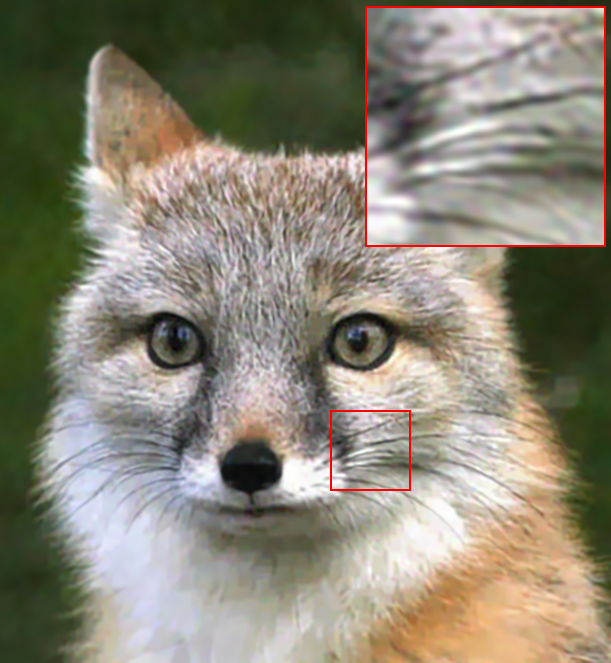}
    \caption*{(26.39 / 0.7586)}
    \label{fig:abl_plain_inr}
\end{subfigure}
\begin{subfigure}{0.32\linewidth}
    \caption{Residual}
    \centering
    \includegraphics[width=\linewidth]{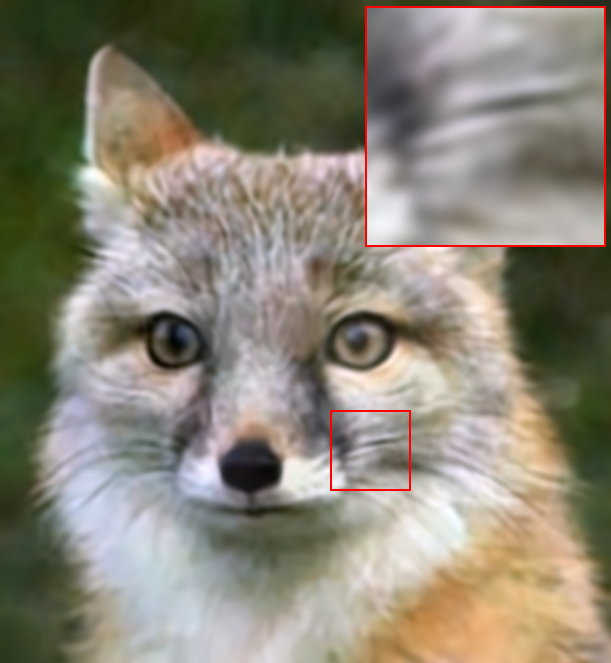}
    \caption*{(23.07 / 0.6278)}
    \label{fig:abl_res_con}
\end{subfigure}
\begin{subfigure}{0.32\linewidth}
    \caption{Multi-grade}
    \centering
    \includegraphics[width=\linewidth]{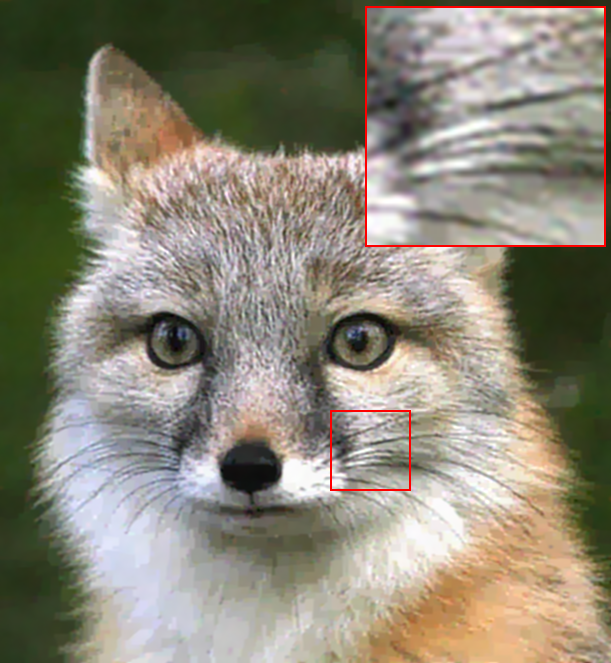}
    \caption*{\textbf{(26.87 / 0.7823)}}
    \label{fig:abl_mgdl}
\end{subfigure}
\caption{Effectiveness of the proposed multi-grade strategy.}
\label{fig:MGDL ablation}
\end{figure}

\section{Conclusion}\label{sec:conclusion}
We introduced MG-SpaIR, a training-data-free image restoration framework that unifies denoising, deblurring, super-resolution, and inpainting within a single self-supervised optimization model. MG-SpaIR combines a multi-grade implicit representation—trained sequentially in a coarse-to-fine residual manner—with a sparse proximal regularization applied in the high-resolution image domain to stabilize fitting and suppress INR-induced artifacts. The multi-grade construction improves representational fidelity and optimization stability by decomposing training into a sequence of smaller subproblems, enabling sharper reconstructions and reduced sensitivity to learning-rate choices. The sparse proximal prior further steers the optimization toward clean, structurally faithful solutions, reducing spurious high-frequency artifacts and residual noise. Experiments on mixed-degradation settings demonstrate that MG-SpaIR achieves state-of-the-art performance among training-data-free INR-based methods and offers a stable, interpretable alternative to data-driven restoration. 

\section*{Acknowledgments}
L. Shen is supported in part by the US National Science Foundation under grant DMS-2208385. 
Y. Xu is supported in part by the US National Science Foundation under grant DMS-2208386, and by the US National Institutes of Health under grant R21CA263876.

\section*{Data Availability}
The datasets analysed during the current study are publicly available. 
The Set5 and Set14 datasets are standard benchmarks in the field and are widely available in public repositories (e.g., \url{https://www.kaggle.com/datasets/ll01dm/set-5-14-super-resolution-dataset}). 
The Flickr2K dataset was introduced by Lim et al. and is available for download in public repositories (e.g., \url{https://www.kaggle.com/datasets/daehoyang/flickr2k}). 
Additionally, the specific image ``The quick brown fox'' used in this study is available on Wikimedia Commons (\url{https://commons.wikimedia.org/wiki/File:The_quick_brown_fox....._(15677707699).jpg}).

\begin{appendices}
\section{Implementation Details: INR Backbones and Proximal Operators}
This section presents the implementation details,
including the INR backbones and the closed-form proximal
operators required by our optimization algorithm.
\subsection{Layer Functions in INR Backbones}
\label{sec:layer_func}
The proposed framework can be directly applied to various foundational INR backbones~\cite{tancik2020fourier,ramasinghe2022beyond,muller2022instant} for image restoration. 
In particular, we adopt two representative backbones, WIRE~\cite{saragadam2023wire} and SIREN~\cite{sitzmann2020implicit}, as illustrative backbones in our experiments. 
WIRE introduces localized frequency modulation inspired by wavelets, while SIREN captures globally coherent high-frequency details through sinusoidal activations. 
Both backbones have been explicitly evaluated on image generalization tasks in their original works and achieved state-of-the-art performance, making them ideal backbones for our framework.

\subsubsection{WIRE.} The WIRE architecture \cite{saragadam2023wire} employs a layer function based on the continuous complex 2D Gabor wavelet. For the WIRE network, each hidden layer is defined as 
\begin{equation}\label{WIRE-hidden-layer}
\begin{aligned}
\phi^{(l)}_{\theta_l}(\mathbf{a}^{(l-1)}) = \sigma(\mathbf{W}_1^{(l)} \mathbf{a}^{l-1} + \mathbf{b}_1^{(l)}; \omega_0, s_0) \cdot \\
e^{-|s_0(\mathbf{W}_2^{(l)} \mathbf{a}^{l-1} + \mathbf{b}_2^{(l)})|^2},
\end{aligned}
\end{equation}
where $\sigma(x; \omega_0, s_0) = e^{j\omega_0 x} e^{-|s_0 x|^2}$,
in which $\omega_0$ and $s_0$ control the frequency and scale of the complex Gabor wavelet. 
In that case, $\theta_l=\{\mathbf{W}_1^{(l)},\mathbf{b}_1^{(l)},\mathbf{W}_2^{(l)},\mathbf{b}_2^{(l)}\}$.
The output layer is defined as
\begin{equation}\label{WIRE-output-layer}
    \psi_{\theta'}(\mathbf{a}^{L}) = \mathbf{W}' \mathbf{a}^{(L)} + \mathbf{b}'
\end{equation}
with parameters $\theta' = \{\mathbf{W}', \mathbf{b}'\}$.

\subsubsection{SIREN.}
SIREN~\cite{sitzmann2020implicit} is a Multi-Layer Perceptron (MLP) with sinusoidal activations, 
demonstrated to effectively represent complex and high-frequency image structures.
For the SIREN network, each hidden layer is defined as 
\begin{equation}\label{SIREN-hidden-layer}
    \phi_{\theta_l}^{(l)}(\mathbf{a}^{(l-1)}) = \sigma_l(\mathbf{W}^{(l)} \mathbf{a}^{(l-1)} + \mathbf{b}^{(l)}),
\end{equation}
where $\theta_l = \{\mathbf{W}^{(l)}, \mathbf{b}^{(l)}\}$ and $\sigma_l$ denotes the activation function at layer $l$. To capture high-frequency variations, the activation functions are chosen as 
\begin{equation}
    \sigma_1(z) = \sin(w_0 z), \quad 
    \sigma_l(z) = \sin(z), \; l \geq 2.  
\end{equation}
where $w_0$ is a scaling constant applied to the first layer to enhance the representation of fine details. Finally, the output layer is defined as
\begin{equation}\label{SIREN-output-layer}
    \psi_{\theta'}(\mathbf{a}^{L}) = \mathbf{W}' \mathbf{a}^{(L)} + \mathbf{b}'
\end{equation}
with parameters $\theta' = \{\mathbf{W}', \mathbf{b}'\}$.

\subsection{Closed-form Expressions of Proximal Operators}

\label{sec:prox}

In the proximal update of \cref{eq:prox_update}, each element of $\mathbf{U}$ is updated via the proximal operator of $\Phi$.
To clarify this step and to facilitate implementation, we provide the explicit closed-form expressions of the proximal mappings corresponding to common sparse regularizers used in our framework.
For a scalar variable $v \in \mathbb{R}$, the proximal mapping associated with $\Phi$ is defined as
\begin{equation}
\operatorname{prox}_{\frac{\gamma\lambda}{\beta}\Phi}(v)
\coloneqq \arg\min_{u \in \mathbb{R}}
\frac{1}{2}(u-v)^2 + \frac{\gamma\lambda}{\beta}\Phi(u).
\end{equation}
\noindent
\textit{Note.} The proximal mapping is, in general, set-valued.
In practice, we use its standard single-valued form (e.g., hard or soft thresholding) to ensure deterministic element-wise updates of $\mathbf{U}$.

We next summarize the closed-form proximal mappings for three commonly used regularizers—$\ell_0$, $\ell_1$, and MCP—which correspond to hard, soft, and firm thresholding, respectively.
These operators are essential in our sparse regularization module and are applied element-wise to the tensor $\mathbf{U}$ during each iteration of the optimization.

\medskip
\noindent
\textbf{(a) $\ell_0$ regularization (hard thresholding):}
\begin{equation}
\Phi(u) = \mathbb{I}[u \neq 0], \qquad
\operatorname{prox}_{\frac{\gamma\lambda}{\beta}\Phi}(v) =
\begin{cases}
0, & |v| \le \sqrt{\frac{2\gamma\lambda}{\beta}},\\
v, & |v| > \sqrt{\frac{2\gamma\lambda}{\beta}}.
\end{cases}
\end{equation}

\medskip
\noindent
\textbf{(b) $\ell_1$ regularization (soft thresholding):}
\begin{equation}
\Phi(u) = |u|, \qquad
\operatorname{prox}_{\frac{\gamma\lambda}{\beta}\Phi}(v) =
\operatorname{sign}(v)\max\{|v| - \frac{\gamma\lambda}{\beta},0\}.
\end{equation}

\medskip
\noindent
\textbf{(c) MCP regularization (firm thresholding):}
With shape parameter $\gamma_\text{shape} > 0$,
\begin{equation}
\Phi(u) =
\begin{cases}
|u| - \dfrac{u^2}{2\gamma_\text{shape}}, & |u| \le \gamma_\text{shape},\\
\dfrac{\gamma_\text{shape}}{2}, & |u| > \gamma_\text{shape},
\end{cases}
\end{equation}
and the corresponding proximal operator is
\begin{equation}
\operatorname{prox}_{\frac{\gamma\lambda}{\beta}\Phi}(v) =
\begin{cases}
0, & |v| \le \frac{\gamma\lambda}{\beta}\\
\dfrac{\operatorname{sign}(v),(|v| - \frac{\gamma\lambda}{\beta})}{1 - (\frac{\gamma\lambda}{\beta})/\gamma_\text{shape}}, 
& \frac{\gamma\lambda}{\beta} < |v| \le \gamma_\text{shape},\\
v, & |v| > \gamma_\text{shape},
\end{cases}
\end{equation}
where $\frac{\gamma\lambda}{\beta} < \gamma_\text{shape}$ ensures $1 - (\frac{\gamma\lambda}{\beta})/\gamma_\text{shape} > 0$. More detailed description on the proximity operator of  the MCP can be found in \cite{Shen-Suter-Tripp:JOTA:2019}.

\noindent
For a tensor input $\mathbf{V}$ of the same dimension as $\mathbf{U}$, 
the proximal mappings of $\ell_0$, $\ell_1$, and MCP are applied element-wise, so that each spatial position and channel in $\mathbf{V}$ is processed independently. 
This element-wise property allows efficient and parallel computation of the updates during optimization.

\section{Analysis of MG-SpaIR Algorithm}\label{Supplement: MG-SpaIR Algorithm}
We first establish a necessary condition characterizing all global minimizers of~\cref{eq:original} under the assumptions that $\mathcal{L}_\text{fid}$ is MSE loss, $\Phi$ is $\ell_0$ regularization, and $\operatorname{OptStep}$ is a gradient descent update with step size $\eta$. Building on this condition, we derive a proximal–gradient update tailored to the model’s sparsity structure and prove its convergence. Finally, we extend these results to the multi-grade setting to characterize the convergence properties of \cref{alg}.

\subsection{Characterizing Global Minimizers: A Necessary Condition}

We first derive a necessary condition for the global minimizers of \cref{eq:original} under a single-grade setting. Consider a neural network with $L$ hidden layers. For convenience, we vectorize the auxiliary variable $\mathbf{U}$, the network parameters $\theta$, and the network output $f_{\theta}(\mathbf{X}_{\text{high}})$.

Define
\begin{equation}
\begin{aligned}
d_1 &:= H_{\text{high}} W_{\text{high}} (2C),\\
d_2 &:= \text{(number of trainable parameters)},\\
d_3 &:= H_{\text{high}} W_{\text{high}} C.
\end{aligned}
\end{equation}
Let $\bm{u}\in\mathbb{R}^{d_1}$ be the vectorized form of $\mathbf{U}$ obtained by stacking all tensor columns.
For each layer $l$, let $\mathbf{W}_l$ and $\mathbf{b}_l$ denote the weight matrix and bias vector, and let $\bm{W}_l$ and $\bm{b}_l$ be their vectorized forms.
The full network parameter vector is
\begin{equation}
\bm{\theta}=(\bm{W}_1^T,\bm{b}_1^T,\ldots,\bm{W}_L^T,\bm{b}_L^T)^T\in\mathbb{R}^{d_2}.
\end{equation}
The output $f_{\theta}(\mathbf{X}_{\text{high}})$ is likewise vectorized into $\bm{f}(\bm{\theta})\in\mathbb{R}^{d_3}$.

Because $\mathbf{M}$, $\mathcal{S}_s$, $\mathcal{A}$, and $\mathcal{B}$ are linear operators, there exist matrices $\mathbf{A}$ and $\mathbf{B}$ such that
$\mathbf{A}\bm{f}(\bm{\theta})$ and $\mathbf{B}\bm{f}(\bm{\theta})$ correspond to the vectorized forms of $\mathbf{M}\odot(\mathcal{S}_s\circ\mathcal{A})$ and $\mathcal{B}$ applied to the network output.
Let $\bm{g}$ be the vectorized $\tilde{\mathbf{Y}}$.
Then \cref{eq:original} is equivalent to
\begin{equation}\label{eq:model:vector}
\begin{aligned}
\min_{(\bm{u}, \bm\theta) \in \mathbb{R}^{d_1} \times \mathbb{R}^{d_2}}&\mathcal{L}(\bm{u}, \bm\theta):=\frac{1}{2}\|\bm{g} - \mathbf{A}\bm{f}(\bm\theta)\|_2^2 \\
&
+ \frac{\beta}{2}\|\bm{u} - \mathbf{B}\bm{f}(\bm\theta)\|_2^2
+ \lambda \|\bm{u}\|_0
\end{aligned},
\end{equation}
where $\|\cdot\|_2$ denotes the Euclidean norm and $\|\cdot\|_0$ the $\ell_0$ norm. For a proper, lower semi-continuous function $h: \mathbb{R}^{d_2} \to \mathbb{R}\cup\{\infty\}$ and $\gamma>0$, the proximal operator of $h$ at $\bm{u}\in\mathbb{R}^{d_1}$ is
\[
    \mathrm{prox}_{\gamma h}(\bm{u}) := \mathrm{argmin}_{\bm{v}\in\mathbb{R}^{d_1}}
\left\{\frac{1}{2\gamma}\|\bm{v}-\bm{u}\|_2^2 + h(\bm{v})\right\}.
\]

The following lemma is instrumental for characterizing the global minimizers of \cref{eq:model:vector}.

\begin{lem}\label{lemma_inequality1}
For $\gamma \in (0,1]$ and $\bm{a},\bm{b} \in \mathbb{R}^{d_1}$,
\begin{equation}\label{TheClaim}
\frac{1}{\gamma}\|\gamma\bm{a}\|_2^2 - \|\bm{a}\|_2^2 \;\leq\; 
\frac{1}{\gamma}\|\bm{b} - (1-\gamma)\bm{a}\|_2^2 - \|\bm{b}\|_2^2.
\end{equation}
\end{lem}
\begin{proof}
The left-hand side equals $(\gamma-1)\|\bm{a}\|_2^2$. Thus, it suffices to show
\begin{equation}
\|\bm{b} - (1-\gamma)\bm{a}\|_2^2 - \gamma\|\bm{b}\|_2^2 - \gamma(\gamma-1)\|\bm{a}\|_2^2 \geq 0.
\end{equation}
Expanding the squares using the inner product $\langle\cdot,\cdot\rangle$ on $\mathbb{R}^{d_1}$ and combining similar terms leads to
\begin{equation}
\begin{aligned}
&\|\bm{b} - (1-\gamma)\bm{a}\|_2^2 - \gamma\|\bm{b}\|_2^2 - \gamma(\gamma-1)\|\bm{a}\|_2^2\\
&= (1-\gamma)(\|\bm{b}\|_2^2 - 2\langle \bm{b},\bm{a}\rangle + \|\bm{a}\|_2^2) \\
&= (1-\gamma)\|\bm{b}-\bm{a}\|_2^2 \;\geq\; 0,
\end{aligned}
\end{equation}
since $\gamma \in (0,1]$. This proves \cref{TheClaim}.
\end{proof} 

Using the lemma above, we now establish a necessary condition for the global minimizers of \cref{eq:model:vector}.
\begin{theo}
\label{theorem:necessary_condition}
If the activation function of the neural network is differentiable, and let $\lambda,\beta>0$.  
If $(\bm{u}^*, \bm\theta^*) \in \mathbb{R}^{d_1}\times\mathbb{R}^{d_2}$ is a global minimizer of \cref{eq:model:vector}, then for all $\gamma \in (0,1]$,
\begin{equation}\label{eq:characterization}
\begin{aligned}
&\bm{u}^* \in \mathrm{prox}_{\frac{\gamma\lambda}{\beta}\|\cdot\|_0}\!\left(\gamma\mathbf{B}\bm{f}(\bm\theta^*) + (1-\gamma)\bm{u}^*\right), \\
&\frac{\partial\mathcal{L}}{\partial\bm\theta}(\bm{u}^*, \bm\theta^*) = 0.
\end{aligned}
\end{equation}
\end{theo}
\begin{proof}
Since $(\bm u^*,\theta^*)$ is a global minimizer of \cref{eq:model:vector}, we have
\begin{equation}\label{eq:global_min_def}
\mathcal L(\bm u^*,\theta^*) \le \mathcal L(\bm u,\theta)
\quad\text{for all }(\bm u,\theta)\in\mathbb{R}^{d_1}\times\mathbb{R}^{d_2}.
\end{equation}
The proof proceeds in two main steps.

\medskip\noindent\textbf{Step 1: Proximal inclusion for \(\bm u^*\).}
Fix $\theta=\theta^*$ in \cref{eq:global_min_def}. Writing out the objective terms that depend on \(\bm u\) yields
\[
\frac{\beta}{2}\|\bm u^*-\mathbf B\bm f(\theta^*)\|_2^2 + \lambda\|\bm u^*\|_0
\le
\frac{\beta}{2}\|\bm u-\mathbf B\bm f(\theta^*)\|_2^2 + \lambda\|\bm u\|_0,
\]
for all $\bm{u}\in\mathbb{R}^{d_1}$. Since $\lambda>0$,
\begin{equation}\label{eq:ineq_scaled}
\frac{\beta}{2\lambda}\|\bm u^*-\mathbf B\bm f(\theta^*)\|_2^2 + \|\bm u^*\|_0
\le
\frac{\beta}{2\lambda}\|\bm u-\mathbf B\bm f(\theta^*)\|_2^2 + \|\bm u\|_0.
\end{equation}

Apply Lemma \cref{lemma_inequality1} with
\[
\bm a := \bm u^* - \mathbf B\bm f(\theta^*), \qquad
\bm b := \bm u - \mathbf B\bm f(\theta^*),
\]
and with the same $\gamma\in(0,1]$ that appears in the theorem. Lemma \cref{lemma_inequality1} gives
\[
\frac{1}{\gamma}\|\gamma\bm a\|_2^2 - \|\bm a\|_2^2
\le
\frac{1}{\gamma}\|\bm b - (1-\gamma)\bm a\|_2^2 - \|\bm b\|_2^2.
\]
Substituting the definitions of \(\bm a,\bm b\) and multiplying the whole inequality by \(\tfrac{\beta}{2\lambda}\) yields
\begin{equation}\label{eq:lemma_applied}
\begin{aligned}
&\frac{\beta}{2\gamma\lambda}\|\gamma(\bm u^*-\mathbf B\bm f(\theta^*))\|_2^2
- \frac{\beta}{2\lambda}\|\bm u^*-\mathbf B\bm f(\theta^*)\|_2^2\\
&\le
\frac{\beta}{2\gamma\lambda}\|\bm u-\mathbf B\bm f(\theta^*)-(1-\gamma)(\bm u^*-\mathbf B\bm f(\theta^*))\|_2^2\\
&\quad - \frac{\beta}{2\lambda}\|\bm u-\mathbf B\bm f(\theta^*)\|_2^2.
\end{aligned}
\end{equation}
Adding inequality \cref{eq:ineq_scaled} to \cref{eq:lemma_applied}, with the negative \(\tfrac{\beta}{2\lambda}\|\cdot\|_2^2\) terms cancel, leads to 
\[
\begin{aligned}
&\frac{\beta}{2\gamma\lambda}\|\gamma(\bm u^*-\mathbf B\bm f(\theta^*))\|_2^2 + \|\bm u^*\|_0\\
&\le
\frac{\beta}{2\gamma\lambda}\|\bm u-\mathbf B\bm f(\theta^*)-(1-\gamma)(\bm u^*-\mathbf B\bm f(\theta^*))\|_2^2\\
&\quad + \|\bm u\|_0,
\end{aligned}
\]
for all $\bm u \in \mathbb{R}^{d_1}$. Rewriting the squared norms inside each argument gives
\[
\begin{aligned}
&\frac{\beta}{2\gamma\lambda}\|\bm u^* - \gamma\mathbf B\bm f(\theta^*) - (1-\gamma)\bm u^*\|_2^2 + \|\bm u^*\|_0\\
&\le
\frac{\beta}{2\gamma\lambda}\|\bm u - \gamma\mathbf B\bm f(\theta^*) - (1-\gamma)\bm u^*\|_2^2 + \|\bm u\|_0.
\end{aligned}
\]
This inequality holds for every \(\bm u\in \mathbb{R}^{d_1}\). Define
\[
\tau := \frac{\gamma\lambda}{\beta} > 0
\qquad\text{so that}\qquad
\frac{\beta}{2\gamma\lambda} = \frac{1}{2\tau}.
\]
Then the last inequality is exactly the statement that \(\bm u^*\) minimizes the function
\[
\bm v \mapsto \frac{1}{2\tau}\|\bm v - (\gamma\mathbf B\bm f(\theta^*) + (1-\gamma)\bm u^*)\|_2^2 + \|\bm v\|_0
\]
over \(\bm v\in\mathbb{R}^{d_1}\). By the definition of the proximal operator (applied to \(h(\cdot)=\|\cdot\|_0\)), this yields the proximal inclusion
\[
\bm u^* \in \operatorname{prox}_{\frac{\gamma\lambda}{\beta}\|\cdot\|_0}\!\big(\gamma\mathbf B\bm f(\theta^*) + (1-\gamma)\bm u^*\big).
\]

\medskip\noindent\textbf{Step 2: Stationarity for \(\theta^*\).}
Fix \(\bm u=\bm u^*\) in \cref{eq:global_min_def}. Then
\[
\mathcal L(\bm u^*,\theta^*) \le \mathcal L(\bm u^*,\theta)
\quad\text{for all }\theta\in\mathbb{R}^{d_2},
\]
so \(\theta^*\) is a global minimizer of the unconstrained differentiable function \(\theta\mapsto\mathcal L(\bm u^*,\theta)\). The differentiability follows from the assumed differentiability of the activation (hence of \(\bm f(\theta)\)) and the linearity of \(\mathbf B\) and \(\mathbf A\); thus \(\mathcal L(\bm u^*,\cdot)\) is differentiable on \(\mathbb R^{d_2}\). By Fermat's rule for unconstrained differentiable minimization, any global minimizer must satisfy the first-order condition
\[
\frac{\partial\mathcal L}{\partial\theta}(\bm u^*,\theta^*) = 0.
\]

Combining Steps 1 and 2 establishes \cref{eq:characterization} for every \(\gamma\in(0,1]\), completing the proof.
\end{proof}


\subsection{Proximal-Gradient Scheme}
Motivated by \cref{theorem:necessary_condition}, we employ a proximal–gradient scheme for solving \cref{eq:model:vector}. The scheme alternates between updating the auxiliary variable $\bm{u}$ via a proximal step and updating the network parameters $\bm\theta$ via gradient descent:
\begin{equation}\label{eq_proximal_gradient}
\begin{aligned}
&\bm{u}^{k+1} \in \mathrm{prox}_{\tfrac{\gamma\lambda}{\beta}\|\cdot\|_0}\!\Big(\gamma \mathbf{B}\bm{f}(\bm\theta^k) + (1 - \gamma)\bm{u}^k\Big),\\
&\bm\theta^{k+1} := \bm\theta^{k} - \eta\,\frac{\partial \mathcal{L}}{\partial \bm\theta}(\bm{u}^{k+1}, \bm\theta^k),
\end{aligned}
\end{equation}
where $\eta>0$ is the step size for the gradient update of $\bm\theta$.


The proximal–gradient scheme can be interpreted as a two-step iterative process: 
\begin{enumerate}
    \item The $\bm{u}$-update enforces sparsity via hard thresholding, promoting a concise representation of the auxiliary variable.
    \item The $\bm\theta$-update adjusts the network parameters to reduce the overall loss, taking into account the current sparse approximation $\bm{u}^{k+1}$.
\end{enumerate}

\subsection{Convergence Analysis}
We analyze the convergence property of the proximal--gradient method in \cref{eq_proximal_gradient}.
To ensure that all quantities involved in the analysis remain well-defined, we assume the existence of a convex and compact set
$\Omega \subset \mathbb{R}^{d_1} \times \mathbb{R}^{d_2}$
and some $\eta_0 > 0$ such that both sequences
${(\bm{u}^k , \theta^k)}_{k=0}^\infty$ and
${(\bm{u}^{k+1}, \theta^k)}_{k=0}^\infty$
generated by the algorithm remain in $\Omega$ for all step sizes $\eta \in (0,\eta_0)$.
This boundedness assumption ensures, in particular, that the loss landscape is explored only within a compact region.

Since the neural network activation function is twice continuously differentiable, the mapping $\theta \mapsto \mathcal{L}(\bm{u},\theta)$ is twice continuously differentiable for every fixed $\bm{u}$.
Let $\bm{H}_{\mathcal{L}}(\bm{u},\theta)$ denote the Hessian of $\mathcal{L}$ with respect to $\theta$, and define
\begin{equation}\label{eq_alpha_hessian}
\alpha := \sup\bigl\{\Vert \bm{H}_\mathcal{L}(\bm{u}, \theta)\Vert_2 : (\bm{u}, \theta) \in \Omega \bigr\},
\end{equation}
where $\Vert\cdot\Vert_2$ is the spectral norm.
The finiteness of $\alpha$ follows from the compactness of $\Omega$.

The next lemma establishes a key descent property for the proximal--gradient update and serves as the main building block of the convergence proof.





\begin{lem}
\label{lemma_inequality2}
Let $\{(\bm{u}^k, \theta^k)\}_{k=0}^\infty$ be generated by \cref{eq_proximal_gradient} with initial point $(\bm{u}^0, \theta^0)$. Then for all $k \geq 1$,
\begin{equation}
\mathcal{L}(\bm{u}^{k+1}, \theta^k) + \frac{\beta}{2}\frac{1-\gamma}{\gamma}\|\bm{u}^{k+1} - \bm{u}^k\|_2^2 
\leq \mathcal{L}(\bm{u}^k, \theta^k).
\end{equation}
\end{lem}
\begin{proof}
By the definition of proximal operators, the first update in \cref{eq_proximal_gradient} is equivalent to
\begin{equation}\label{eq_proof_lemma_1}
\begin{aligned}
\bm{u}^{k+1} \in &{\rm argmin}_{\bm{u} \in \mathbb{R}^{d_1}} 
\{ \tfrac{1}{2}\tfrac{\beta}{\gamma\lambda}
\|\bm{u} - \gamma\bm{B}\bm{f}(\theta^k) - (1-\gamma)\bm{u}^k\|_2^2\\
&\qquad \qquad \qquad+ \|\bm{u}\|_0\}.
\end{aligned}
\end{equation}
Expanding the quadratic term yields
\begin{equation}
\begin{aligned}
&\|\bm{u} - \gamma\bm{B}\bm{f}(\theta^k) - (1-\gamma)\bm{u}^k\|_2^2\\
&= \gamma\|\bm{u} - \bm{B}\bm{f}(\theta^k)\|_2^2 
+ (1-\gamma)\|\bm{u} - \bm{u}^k\|_2^2 + r,
\end{aligned}
\end{equation}
where $r$ is independent of $\bm{u}$ and thus irrelevant to minimization. Substituting back, we obtain
\begin{equation}
\begin{aligned}
&\tfrac{1}{2}\tfrac{\beta}{\gamma\lambda}\|\bm{u} - \gamma\bm{B}\bm{f}(\theta^k) - (1-\gamma)\bm{u}^k\|_2^2 + \|\bm{u}\|_0 \\
&= \tfrac{1}{\lambda}\left( \tfrac{\beta}{2}\|\bm{u} - \bm{B}\bm{f}(\theta^k)\|_2^2 
+ \lambda\|\bm{u}\|_0 
+ \tfrac{\beta}{2}\tfrac{1-\gamma}{\gamma}\|\bm{u} - \bm{u}^k\|_2^2 \right)\\
&\quad + c,
\end{aligned}
\end{equation}
with $c$ again independent of $\bm{u}$. By the definition of $\mathcal{L}$, this is equivalent to
\begin{equation}\label{eq_proof_lemma_2}
\bm{u}^{k+1} \in {\rm argmin}_{\bm{u} \in \mathbb{R}^{d_1}}
\left\{ \mathcal{L}(\bm{u}, \theta^k) 
+ \tfrac{\beta}{2}\tfrac{1-\gamma}{\gamma}\|\bm{u} - \bm{u}^k\|_2^2 \right\}.
\end{equation}
Hence $\bm{u}^{k+1}$ satisfies
\begin{equation}
\begin{aligned}
&\mathcal{L}(\bm{u}^{k+1}, \theta^k) 
+ \tfrac{\beta}{2}\tfrac{1-\gamma}{\gamma}\|\bm{u}^{k+1} - \bm{u}^k\|_2^2\\
&\le \mathcal{L}(\bm{u}, \theta^k) 
+ \tfrac{\beta}{2}\tfrac{1-\gamma}{\gamma}\|\bm{u} - \bm{u}^k\|_2^2
\end{aligned}
\end{equation}
for all $\bm{u} \in \mathbb{R}^{d_1}$. Taking $\bm{u} = \bm{u}^k$ gives the claim.
\end{proof}

The following convergence theorem holds under these assumptions.
\begin{theo}
\label{theorem:convergence:algorithm supp}
Let $\{(\bm{u}^k, \theta^k)\}_{k=0}^\infty$ be the sequence generated by the proximal–gradient method \cref{eq_proximal_gradient} with initial guess $(\bm{u}^0, \theta^0)$. Suppose the activation function in the neural network is twice continuously differentiable, and there exists a convex, compact set $\Omega \subset \mathbb{R}^{d_1} \times \mathbb{R}^{d_2}$ such that 
\[
\{(\bm{u}^k, \theta^k)\}_{k=0}^\infty \subset \Omega, 
\quad 
\{(\bm{u}^{k+1}, \theta^k)\}_{k=0}^\infty \subset \Omega.
\]
Assume that $\alpha$ defined in \cref{eq_alpha_hessian} is finite. If the learning rate $\eta \in (0, 2/\alpha)$ and the parameter $\gamma \in (0,1)$, then the following hold:
\begin{enumerate}
    \item $\displaystyle \lim_{k\to\infty}\mathcal{L}(\bm{u}^k, \theta^k) = L^*$ for some $L^* \geq 0$;
    \item $\displaystyle \lim_{k\to\infty}\|\bm{u}^{k+1} - \bm{u}^k\|_2 = 0,$ and $\displaystyle\lim_{k\to\infty}\|\theta^{k+1} - \theta^k\|_2 = 0$;
    \item $\displaystyle \lim_{k\to\infty}\mathrm{dist}\!\left(\bm{0}, \tfrac{\partial\mathcal{L}}{\partial\bm{u}}(\bm{u}^{k+1}, \theta^k)\right) \,\mathord{=}\, 0,$ 
    $\text{and }\displaystyle \lim_{k\to\infty}\tfrac{\partial\mathcal{L}}{\partial\theta}(\bm{u}^{k+1}, \theta^k) = 0$.
\end{enumerate}
Here $\tfrac{\partial\mathcal{L}}{\partial\bm{u}}(\bm{u}^{k+1}, \theta^k)$ denotes the subdifferential of $\mathcal{L}(\cdot, \theta^k)$ at $\bm{u}^{k+1}$, and $\tfrac{\partial\mathcal{L}}{\partial\theta}(\bm{u}^{k+1}, \cdot)$ is the gradient with respect to $\theta$ at $\theta^k$. The distance between a point $\bm{p}\in\mathbb{R}^{d_1}$ and a set $S\subset \mathbb{R}^{d_1}$ is defined as 
\[
\mathrm{dist}(\bm{p}, S) := \inf\{\|\bm{q}-\bm{p}\|_2 : \bm{q}\in S\}.
\]
\end{theo}
\begin{proof}
Since the activation function is twice continuously differentiable, $\mathcal{L}(\bm{u}, \theta)$ is twice continuously differentiable with respect to $\theta$. By the Taylor expansion, there exists some $\bar{\theta}^k$ between $\theta^k$ and $\theta^{k+1}$ such that
\begin{equation}\label{eq_proof_theorem_1}
\begin{aligned}
&\mathcal{L}(\bm{u}^{k+1}, \theta^{k+1})\\
&= \mathcal{L}(\bm{u}^{k+1}, \theta^k) + \frac{\partial \mathcal{L}}{\partial \theta}(\bm{u}^{k+1}, \theta^k)^T (\theta^{k+1} - \theta^k)\\
&\quad + \frac{1}{2} (\theta^{k+1} - \theta^k)^T \bm{H}_\mathcal{L}(\bm{u}^{k+1}, \bar{\theta}^k) (\theta^{k+1} - \theta^k).
\end{aligned}
\end{equation}

From the second step of \cref{eq_proximal_gradient}, we have
\begin{equation}
\frac{\partial \mathcal{L}}{\partial \theta}(\bm{u}^{k+1}, \theta^k) = -\frac{1}{\eta} (\theta^{k+1} - \theta^k).
\end{equation}
Substituting this into \cref{eq_proof_theorem_1} gives
\begin{equation}\label{eq_proof_theorem_2}
\begin{aligned}
&\mathcal{L}(\bm{u}^{k+1}, \theta^{k+1})\\
&= \mathcal{L}(\bm{u}^{k+1}, \theta^k) - \frac{1}{\eta}\|\theta^{k+1} - \theta^k\|_2^2\\
&\quad + \frac{1}{2} (\theta^{k+1} - \theta^k)^T \bm{H}_\mathcal{L}(\bm{u}^{k+1}, \bar{\theta}^k) (\theta^{k+1} - \theta^k).
\end{aligned}
\end{equation}

By convexity of $\Omega$ and the assumption that $(\bm{u}^k, \theta^k), \ (\bm{u}^{k+1}, \theta^k) \in \Omega$, we have $(\bm{u}^{k+1}, \bar{\theta}^k) \in \Omega$. Using the spectral norm bound \cref{eq_alpha_hessian} yields
\begin{equation}
\begin{aligned}
&\left| \frac{1}{2} (\theta^{k+1} - \theta^k)^T \bm{H}_\mathcal{L}(\bm{u}^{k+1}, \bar{\theta}^k) (\theta^{k+1} - \theta^k) \right|\\
&\le \frac{\alpha}{2} \|\theta^{k+1} - \theta^k\|_2^2,
\end{aligned}
\end{equation}
which substituted into \cref{eq_proof_theorem_2} gives
\begin{equation}\label{eq_proof_theorem_3}
\mathcal{L}(\bm{u}^{k+1}, \theta^{k+1}) \le \mathcal{L}(\bm{u}^{k+1}, \theta^k) + \left( \frac{\alpha}{2} - \frac{1}{\eta} \right) \|\theta^{k+1} - \theta^k\|_2^2.
\end{equation}

From Lemma \cref{lemma_inequality2}, we also have
\begin{equation}\label{eq_proof_theorem_4}
\mathcal{L}(\bm{u}^{k+1}, \theta^k) + \frac{\beta}{2}\frac{1-\gamma}{\gamma} \|\bm{u}^{k+1} - \bm{u}^k\|_2^2 \le \mathcal{L}(\bm{u}^k, \theta^k).
\end{equation}
Combining \cref{eq_proof_theorem_3} and \cref{eq_proof_theorem_4} leads to
\begin{equation}\label{eq_proof_theorem_5}
\begin{aligned}
\mathcal{L}(\bm{u}^{k+1}, \theta^{k+1}) &+ \frac{\beta}{2}\frac{1-\gamma}{\gamma} \|\bm{u}^{k+1} - \bm{u}^k\|_2^2\\
&+ \left( \frac{1}{\eta} - \frac{\alpha}{2} \right) \|\theta^{k+1} - \theta^k\|_2^2\\
&\le \mathcal{L}(\bm{u}^k, \theta^k).
\end{aligned}
\end{equation}
Since $\frac{\beta}{2}\frac{1-\gamma}{\gamma} > 0$ and $\frac{1}{\eta} - \frac{\alpha}{2} > 0$, we conclude that $\{\mathcal{L}(\bm{u}^k, \theta^k)\}$ is monotonically nonincreasing and bounded below by zero. Hence, there exists $L^* \ge 0$ such that
\begin{equation}
\lim_{k \to \infty} \mathcal{L}(\bm{u}^k, \theta^k) = L^*,
\end{equation}
proving Item 1. 

Equation \cref{eq_proof_theorem_5} also implies
\begin{equation}
\lim_{k \to \infty} \|\bm{u}^{k+1} - \bm{u}^k\|_2 = 0, \quad 
\lim_{k \to \infty} \|\theta^{k+1} - \theta^k\|_2 = 0,
\end{equation}
which proves Item 2.

For Item 3, from \cref{eq_proof_lemma_2} we have
\begin{equation}
-\frac{\beta(1-\gamma)}{\gamma} (\bm{u}^{k+1} - \bm{u}^k) \in \frac{\partial \mathcal{L}}{\partial \bm{u}}(\bm{u}^{k+1}, \theta^k),
\end{equation}
so
\begin{equation}
\mathrm{dist}\left(\bm{0}, \frac{\partial \mathcal{L}}{\partial \bm{u}}(\bm{u}^{k+1}, \theta^k)\right) \le \frac{\beta(1-\gamma)}{\gamma} \|\bm{u}^{k+1} - \bm{u}^k\|_2 \to 0.
\end{equation}
Similarly, the gradient update gives
\begin{equation}
\frac{\partial \mathcal{L}}{\partial \theta}(\bm{u}^{k+1}, \theta^k) = -\frac{1}{\eta} (\theta^{k+1} - \theta^k) \to 0,
\end{equation}
proving Item 3 and completing the proof.
\end{proof}

\subsection{Extension to the Multi-Grade Setting}

Let $\bm\theta_l$ denote the vectorized parameters of $\theta^{(l)}$, and let $\bm z_l(\bm\theta_l)$ denote the vectorized network output $f^{(l)}_{\theta^{(l)};\Theta_{\le l-1}^*}(\mathbf X_{\text{high}})$ at grade $l$.
Note that the lower-grade parameters $\Theta^{*}_{\le l-1}$ are fixed during grade-$l$ optimization and enter $\mathcal{L}_l$ only as constants. Hence they do not affect the Lipschitz or smoothness properties of the $\theta_l$-gradient.
Let $d_{2,l}$ be the dimension of $\bm\theta_l$. The vectorized optimization model in grade $l$ is
\begin{equation}\label{eq:mgdl:vector}
\begin{aligned}
\min_{(\bm{u}, \bm\theta_l) \in \mathbb{R}^{d_1} \times \mathbb{R}^{d_{2,l}}} &\mathcal{L}_l(\bm{u}, \bm\theta_l) := \frac{1}{2}\|\bm{g} - \mathbf{A}\bm{z}_l(\bm\theta_l)\|_2^2\\
&
+ \frac{\beta}{2}\|\bm{u} - \mathbf{B}\bm{z}_l(\bm\theta_l)\|_2^2 + \lambda \|\bm{u}\|_0.
\end{aligned}
\end{equation}
\cref{theorem:necessary_condition} implies that if a pair $(\bm{u}_l^*, \bm\theta_l^*) \in \mathbb{R}^{d_1} \times \mathbb{R}^{d_{2,l}}$ is a global minimizer of \cref{eq:mgdl:vector}, then there holds that
\[
\begin{aligned}
&\bm{u}_l^* \in \mathrm{prox}_{\frac{\gamma \lambda}{\beta}\|\cdot\|_0}(\gamma \mathbf{B} \bm{z}_l(\bm\theta_l^*) + (1-\gamma) \bm{u}_l^*),\\
&\frac{\partial \mathcal{L}_l}{\partial \bm\theta_l}(\bm{u}_l^*, \bm\theta_l^*) = 0,
\end{aligned}
\]
for $\gamma \in (0,1]$. The proximal-gradient updates in grade $l$ are
\begin{equation}\label{eq:proximal_gradient:mgdl}
\begin{aligned}
\bm{u}_l^{k+1} &\in \mathrm{prox}_{\frac{\gamma_l \lambda}{\beta}\|\cdot\|_0}(\gamma_l \mathbf{B} \bm{z}_l(\bm\theta_l^k) + (1-\gamma_l) \bm{u}_l^k),\\
\bm\theta_l^{k+1} &:= \bm\theta_l^k - \eta_l \frac{\partial \mathcal{L}_l}{\partial \bm\theta_l}(\bm{u}_l^{k+1}, \bm\theta_l^k),
\end{aligned}
\end{equation}
with initial point $(\bm{u}_l^0, \bm\theta_l^0)$.

Assume the existence of a convex and compact set $\Omega_l \subset \mathbb{R}^{d_1} \times \mathbb{R}^{d_{2,l}}$ and $\eta_0 > 0$ such that 
\[
\{(\bm{u}_l^k, \bm\theta_l^k)\}_{k=0}^\infty \subset \Omega_l, \quad 
\{(\bm{u}_l^{k+1}, \bm\theta_l^k)\}_{k=0}^\infty \subset \Omega_l,
\]
for all $\eta \in (0, \eta_0)$ and define
\[
\alpha_l := \sup \{\|\bm{H}_{\mathcal{L}_l}(\bm{u}, \bm\theta_l)\|_2 : (\bm{u}, \bm\theta_l) \in \Omega_l\}.
\]
\noindent
The finiteness of $\alpha_l$ follows from the twice continuous differentiability of the activation functions and the compactness of $\Omega_l$, which together ensure that $H_{\mathcal{L}_l}$ is continuous and therefore bounded on $\Omega_l$.
Since the grade-$l$ update in \cref{eq:proximal_gradient:mgdl} has exactly the same proximal--gradient structure as \cref{eq_proximal_gradient}, the assumptions of \cref{theorem:convergence:algorithm supp} hold for each grade with the corresponding $\Omega_l$ and $\alpha_l$. Consequently, the convergence theorem for the Multi-Grade setting follows directly.


\begin{theo}\label{thm: multi-grade vector}
Let $\{(\bm{u}_l^k, \bm\theta_l^k)\}_{k=0}^\infty$ be generated by \cref{eq:proximal_gradient:mgdl} with initial guess $(\bm{u}_l^0, \bm\theta_l^0)$. Assume the activation function is twice continuously differentiable, $\Omega_l$ exists as above, and $\alpha_l < \infty$. If $\eta_l \in (0, 2/\alpha_l)$ and $\gamma_l \in (0,1)$, then
\begin{enumerate}
\item $\displaystyle\lim_{k\to\infty} \mathcal{L}_l(\bm{u}_l^k, \bm\theta_l^k) = L_l^* \ge 0$;
\item $\displaystyle\lim_{k\to\infty} \|\bm{u}_l^{k+1}-\bm{u}_l^k\|_2 = 0,$ and $\displaystyle\lim_{k\to\infty} \|\bm\theta_l^{k+1}-\bm\theta_l^k\|_2 = 0$;
\item $\displaystyle\lim_{k\to\infty} \mathrm{dist}\left(\bm{0}, \frac{\partial \mathcal{L}_l}{\partial \bm{u}_l}(\bm{u}_l^{k+1}, \bm\theta_l^k)\right) \, \mathord{=} \, 0,$ $\text{ and } \displaystyle\lim_{k\to\infty} \frac{\partial \mathcal{L}_l}{\partial \bm{\theta}_l}(\bm{u}_l^{k+1}, \bm\theta_l^k) = 0$.
\end{enumerate}
Here $\tfrac{\partial\mathcal{L}}{\partial\bm{u}_l}(\bm{u}_l^{k+1}, \theta_l^k)$ denotes the subdifferential of $\mathcal{L}(\cdot, \theta_l^k)$ at $\bm{u}_l^{k+1}$, and $\tfrac{\partial\mathcal{L}}{\partial\theta_l}(\bm{u}_l^{k+1}, \cdot)$ is the gradient with respect to $\theta_l$ at $\theta_l^k$. The distance between a point $\bm{p}\in\mathbb{R}^{d_1}$ and a set $S\subset \mathbb{R}^{d_1}$ is defined as 
\[
\mathrm{dist}(\bm{p}, S) := \inf\{\|\bm{q}-\bm{p}\|_2 : \bm{q}\in S\}.
\]
\end{theo}
\cref{theorem:convergence:algorithm} is merely a matrix-form rewriting of 
\cref{thm: multi-grade vector}, obtained by stacking the variables 
$\{(\bm u_l, \bm\theta_l)\}_l$ into $(\mathbf U, \theta)$. Hence, the result carries over directly.


\subsection{Spectral Analysis}
This subsection derives a spectral stability condition for \cref{alg}. By linearizing the $\theta$-gradient around the previous iterate, the update is rewritten as an affine recursion whose linear part is governed by the Hessian. This leads to the spectral quantity $\tau$, which bounds the norm of the linearized update operator. When $\tau<1$, the recursion becomes contractive, providing a sufficient condition that guarantees stable and convergent parameter dynamics.



Recall the updates from \cref{eq_proximal_gradient}:
\begin{equation*}
\begin{aligned}
&\bm{u}^{k+1} \in \mathrm{prox}_{\tfrac{\gamma\lambda}{\beta}\|\cdot\|_0}\big(\gamma\bm{B}\bm{f}(\theta^k) + (1 - \gamma)\bm{u}^k\big),\\
&\theta^{k+1} := \theta^{k} - \eta\frac{\partial\mathcal{L}}{\partial\theta}(\bm{u}^{k + 1}, \theta^k).
\end{aligned}
\end{equation*}

The second update is a gradient descent step in $\theta$, but with the gradient evaluated at the freshly updated $\bm{u}^{k+1}$. To study convergence, we view this update as a fixed-point (Picard) iteration:
\begin{equation*}
\theta^{k+1} = \big(\mathcal{I} - \eta\,\tfrac{\partial\mathcal{L}}{\partial\theta}(\bm{u}^{k + 1}, \cdot)\big)(\theta^k),
\end{equation*}
where $\mathcal{I}$ denotes the identity operator. Classical Picard theory guarantees convergence when the update operator is nonexpansive, whereas in deep networks the map 
\[
\mathcal{I} - \eta\,\tfrac{\partial\mathcal{L}}{\partial\theta}(\bm{u}^{k + 1}, \cdot)
\]
can be expansive. To analyze the resulting dynamics, we therefore apply a Taylor expansion of the gradient around the previous iterate.

Expanding around $\theta^{k-1}$, we approximate
\begin{equation}
\begin{aligned}
&\frac{\partial\mathcal{L}}{\partial\theta}(\bm{u}^{k + 1}, \theta^k)\\
&= \frac{\partial\mathcal{L}}{\partial\theta}(\bm{u}^{k + 1}, \theta^{k - 1}) + \bm{H}_{\mathcal{L}}(\bm{u}^{k + 1}, \theta^{k-1})(\theta^k - \theta^{k-1})\\
&\quad + \tfrac{1}{2}(\theta^k - \theta^{k-1})^T\bm{T}_{\mathcal{L}}(\bm{u}^{k + 1}, \bar{\theta}^k)(\theta^k - \theta^{k-1}),
\end{aligned}
\end{equation}
where $\bm{H}_{\mathcal{L}}$ and $\bm{T}_{\mathcal{L}}$ denote the Hessian and third order derivative tensor of $\mathcal{L}$ with respect to $\theta$, and $\bar{\theta}^k$ lies between $\theta^{k-1}$ and $\theta^k$.

Define
\begin{equation}
\begin{cases}
\bm{v}^k := -\eta\,\tfrac{\partial\mathcal{L}}{\partial\theta}(\bm{u}^{k + 1}, \theta^{k - 1})  + \eta\,\bm{H}_{\mathcal{L}}(\bm{u}^{k + 1}, \theta^{k - 1})\theta^{k - 1},\\\\
\bm{r}^k := -\tfrac{\eta}{2}(\theta^k - \theta^{k-1})^T\bm{T}_{\mathcal{L}}(\bm{u}^{k + 1}, \bar{\theta}^k)(\theta^k - \theta^{k-1}),\\\\
\bm{C}^k := \bm{I} - \eta\,\bm{H}_{\mathcal{L}}(\bm{u}^{k + 1}, \theta^{k - 1}),
\end{cases}
\end{equation}
with $\bm{I}$ being the identity. The $\theta$-update can then be written as
\begin{equation}\label{eq:gd}
\theta^{k+1} = \bm{C}^k\theta^k + \bm{v}^k + \bm{r}^k.
\end{equation}
Let 
\begin{equation}
\tau := \sup_{(\bm{u}, \theta) \in \Omega}\|\bm{I} - \eta\bm{H}_{\mathcal{L}}(\bm{u}, \theta)\|_2,
\end{equation}
where $\Omega \subset \mathbb{R}^{d_1} \times \mathbb{R}^{d_2}$ is convex and compact and $\bm{I}$ is the identity matrix. The following result characterizes the convergence of $\{\theta^k\}$.

\begin{theo}
\label{theorem:spectral}
Assume the activation function of the neural network is three times continuously differentiable, and let $\Omega \subset \mathbb{R}^{d_1} \times \mathbb{R}^{d_2}$ be a convex and compact set. Suppose the iterates $\{(\bm{u}^k, \theta^k)\}_{k=1}^\infty$ are generated by \cref{eq_proximal_gradient} and satisfy
\[
(\bm{u}^k, \theta^k) \in \Omega,\quad 
(\bm{u}^{k+1}, \theta^k) \in \Omega,\quad 
(\bm{u}^{k+1}, \theta^{k-1}) \in \Omega
\]
for all $k$. If $\tau<1$,
then the parameter sequence $\{\theta^k\}_{k=1}^\infty$ converges.
\end{theo}
\begin{proof} 
From \cref{eq:gd}, we have
\begin{equation}\label{eq:th4p1:1}
\begin{aligned}
\theta^{k+1} &= \Big(\prod_{j=1}^k \bm{C}^j\Big)\theta^1 
+ \sum_{m=1}^k\Big(\prod_{j=m+1}^k \bm{C}^j\Big)\bm{v}^m\\
&\quad + \sum_{m=1}^k\Big(\prod_{j=m+1}^k \bm{C}^j\Big)\bm{r}^m.
\end{aligned}
\end{equation}
To prove convergence of $\{\theta^k\}$, it suffices to show that each of the three terms on the right-hand side of \cref{eq:th4p1:1} converges.

\medskip
\noindent\textbf{1. The homogeneous term.}
By definition of $\tau$, we have 
$$
\|\bm{C}^k\|_2 \leq \tau\ \  \mbox{for all}\ \  k.
$$
Thus
\[
\Big\|\Big(\prod_{j=1}^k \bm{C}^j\Big)\theta^1\Big\|_2 
\leq \|\theta^1\|_2 \prod_{j=1}^k \|\bm{C}^j\|_2 
\leq \tau^k \|\theta^1\|_2.
\]
Since $\tau < 1$, it follows that $\tau^k \to 0$ as $k \to \infty$, and hence
\[
\lim_{k\to\infty}\Big\|\Big(\prod_{j=1}^k \bm{C}^j\Big)\theta^1\Big\|_2 = 0.
\]
Therefore the first term converges to zero.

\medskip
\noindent\textbf{2. The $\bm{v}^m$ term.}
Because the activation is three times continuously differentiable, $\mathcal{L}(\bm{u},\theta)$ is smooth in $\theta$. Although $\mathcal{L}$ is not continuous in $\bm{u}$ due to the $\|\cdot\|_0$ term, this term vanishes in differentiation with respect to $\theta$, so both $\tfrac{\partial \mathcal{L}}{\partial \theta}$ and $\bm{H}_{\mathcal{L}}$ are continuous in $(\bm{u},\theta)$. Compactness of $\Omega$ ensures the existence of a constant $V>0$ such that
\[
\big\| -\eta\tfrac{\partial \mathcal{L}}{\partial \theta}(\bm{u},\theta) + \eta \bm{H}_{\mathcal{L}}(\bm{u},\theta)\theta \big\|_2 \leq V, 
\qquad \mbox{for all}\ \ (\bm{u},\theta)\in\Omega.
\]
By the definition of $\bm{v}^k$, this implies 
\[
\|\bm{v}^k\|_2 \leq V, \ \ \mbox{for all}\ \ k. 
\]
Moreover,
\[
\Big\|\Big(\prod_{j=m+1}^k \bm{C}^j\Big)\bm{v}^m\Big\|_2 
\leq \|\bm{v}^m\|_2 \prod_{j=m+1}^k \|\bm{C}^j\|_2 
\leq V \tau^{k-m}.
\]
Hence
\[
\begin{aligned}
\sum_{m=1}^k \Big\|\Big(\prod_{j=m+1}^k \bm{C}^j\Big)\bm{v}^m\Big\|_2
&\leq V \sum_{m=1}^k \tau^{k-m}\\
&= V \frac{1-\tau^k}{1-\tau}\\
&< \frac{V}{1-\tau}.
\end{aligned}
\]
This shows that the series 
$$
\sum_{m=1}^\infty \left(\prod_{j=m+1}^\infty \bm{C}^j\right)\bm{v}^m
$$ 
converges absolutely, and hence the second term converges.

\medskip
\noindent\textbf{3. The $\bm{r}^m$ term.}
Since $\mathcal{L}(\bm{u},\theta)$ is three times continuously differentiable, $\bm{T}_{\mathcal{L}}(\bm{u},\theta)$ is continuous in $(\bm{u},\theta)$. By compactness of $\Omega$, there exists $T>0$ such that $\|\bm{T}_{\mathcal{L}}(\bm{u},\theta)\|_2 \leq T$ for all $(\bm{u},\theta)\in\Omega$. By convexity of $\Omega$, the points $(\bm{u}^{k+1},\bar{\theta}^k)$ also lie in $\Omega$, so
\[
\|\bm{T}_{\mathcal{L}}(\bm{u}^{k+1},\bar{\theta}^k)\|_2 \leq T, \qquad \forall k.
\]
By definition of $\bm{r}^k$, we obtain
\[
\begin{aligned}
\|\bm{r}^k\|_2 &\leq \frac{\eta}{2}\|\bm{T}_{\mathcal{L}}(\bm{u}^{k+1},\bar{\theta}^k)\|_2 \|\theta^k - \theta^{k-1}\|_2^2\\
&\leq \frac{\eta T}{2}\|\theta^k - \theta^{k-1}\|_2^2.
\end{aligned}
\]
Since $\{(\bm{u}^k,\theta^k)\}$ is bounded, there exists $U>0$ such that $\|\theta^k - \theta^{k-1}\|_2 \leq U$ for all $k$. Therefore,
\[
\|\bm{r}^k\|_2 \leq \frac{\eta T U^2}{2}, \qquad \forall k.
\]
Following the same argument as for the $\bm{v}^m$ term, we conclude that 
$$
\sum_{m=1}^\infty \left(\prod_{j=m+1}^\infty \bm{C}^j\right)\bm{r}^m
$$ 
is absolutely convergent, and hence the third term converges.

Since all three components in \cref{eq:th4p1:1} converge, the sequence $\{\theta^k\}$ converges as claimed.
\end{proof}

\section{Experiment Details and More Experiments}
\label{sec:supp_exp}

\subsection{Pipeline Details}
To establish a competitive and transparent classical baseline, 
we constructed a traditional image restoration pipeline that 
explicitly inverts each component of the degradation model 
defined in \cref{eq:degradation}. This pipeline serves as a representative classical baseline that explicitly inverts each degradation component, providing a contrast to our implicit, training-data-free INR-based approach.
The pipeline consists of four sequential stages:

\begin{enumerate}
\item \textbf{Inpainting:} Missing pixels are first reconstructed 
using the fast marching-based method of Telea~\cite{telea2004image}, 
which effectively propagates image structures from known to unknown regions.

\item \textbf{Denoising:} The inpainted result is then denoised using 
BM3D~\cite{dabov2007image}, a state-of-the-art method renowned for its 
performance in removing additive white Gaussian noise while preserving details.

\item \textbf{Super-Resolution:} The denoised low-resolution image is upscaled via bicubic interpolation~\cite{keys2003cubic}. 
We adopt bicubic here to match the bicubic downsampling used in the degradation model, ensuring that the upsampling operator faithfully inverts the resolution reduction.

\item \textbf{Deblurring:} Finally, a classical Wiener filter~\cite{10.7551/mitpress/2946.001.0001} 
is applied to remove blur in the high-resolution domain.
\end{enumerate}

To ensure this pipeline performs at its optimal capacity for a fair comparison, all parameters were carefully selected. The blur kernel and noise variance, when known from the degradation parameters, were provided directly to the Wiener filter and BM3D, respectively. 
For all remaining parameters, we used the default or commonly recommended settings from the original papers, as these configurations are well established to yield robust and high-quality results.

\subsection{Necessity of High-Resolution Regularization (Case Study on WIRE)}
\label{sec:WIRE reg}
To assess the necessity of explicit regularization and the effect of imposing it in the high-resolution image domain, we conduct a case study on WIRE~\cite{saragadam2023wire}.
As shown in \cref{fig:wire_super,fig:wire_inpainting},
adding our high-resolution regularization consistently improves super-resolution and inpainting quality, confirming the necessity of the proposed regularization for image restoration across different INR models.
\begin{figure}[htbp]
    \begin{subfigure}{0.32\linewidth}
        \caption{Input ($2\times\downarrow$)}
        \centering
        \includegraphics[width=\linewidth]{butterfly_degraded_blur0.0_down2_noisegaussian_std0.0_sp0.0_patches0_prob0.0_PSNR26.62_SSIM0.9109_zoom.png}
        \caption*{\quad}
    \end{subfigure}
    \begin{subfigure}{0.32\linewidth}
        \caption{WO/reg}
        \centering
        \includegraphics[width=\linewidth]{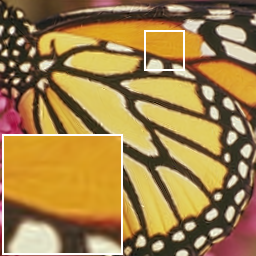}
        \caption*{(24.98 / 0.8903)}
    \end{subfigure}
    \begin{subfigure}{0.32\linewidth}
        \caption{Small low-res reg}
        \centering
        \includegraphics[width=\linewidth]{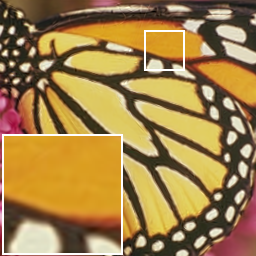}
        \caption*{(24.83 / 0.8910)}
    \end{subfigure}
    \begin{subfigure}{0.32\linewidth}
        \caption{Large low-res reg}
        \centering
        \includegraphics[width=\linewidth]{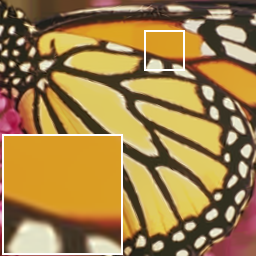}
        \caption*{(24.96 / 0.8810)}
    \end{subfigure}
    \begin{subfigure}{0.32\linewidth}
        \caption{Small high-res reg}
        \centering
        \includegraphics[width=\linewidth]{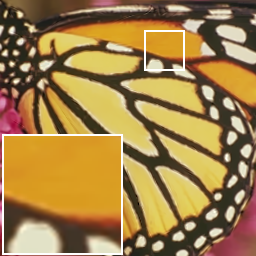}
        \caption*{\textbf{(25.36 / 0.9013)}}
    \end{subfigure}
    \begin{subfigure}{0.32\linewidth}
            \caption{Ground truth}
        \centering
        \includegraphics[width=\linewidth]{butterfly_zoom.png}
        \caption*{\quad}
    \end{subfigure}
    \caption{Comparison of results using different regularizations for the super-resolution task with WIRE.}
 \label{fig:wire_super}
\end{figure}

\begin{figure}[htbp]
    \begin{subfigure}{0.32\linewidth}
        \caption{Input ($50\%$ mask)}
        \centering
        \includegraphics[width=\linewidth]{butterfly_degraded_blur0.0_down1_noisegaussian_std0.0_sp0.0_patches0_prob0.5_PSNR8.54_SSIM0.1934_zoom.png}
        \caption*{\quad}
    \end{subfigure}
    \begin{subfigure}{0.32\linewidth}
        \caption{WO/reg}
        \centering
        \includegraphics[width=\linewidth]{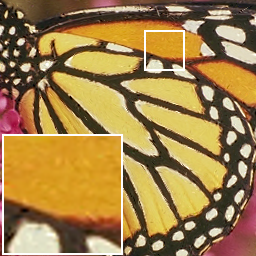}
        \caption*{(27.19 / 0.9115)}
    \end{subfigure}
    \begin{subfigure}{0.32\linewidth}
        \caption{W/reg}
        \centering
        \includegraphics[width=\linewidth]{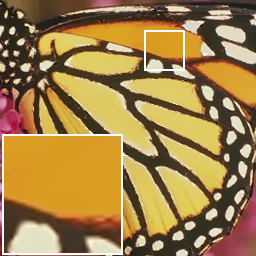}
        \caption*{\textbf{(28.74 / 0.9402)}}
    \end{subfigure}
    \caption{Comparison of inpainting results with and without regularization using WIRE.}
    \label{fig:wire_inpainting}
\end{figure}

\end{appendices}


\bibliography{sn-bibliography}

@article{fang2025computational,
  title={Computational advantages of multi-grade deep learning: Convergence analysis and performance insights},
  author={Fang, Ronglong and Xu, Yuesheng},
  journal={arXiv preprint arXiv:2507.20351},
  year={2025}
}

@article{fang2024addressing,
  title={Addressing spectral bias of deep neural networks by multi-grade deep learning},
  author={Fang, Ronglong and Xu, Yuesheng},
  journal={Advances in Neural Information Processing Systems},
  volume={37},
  pages={114122--114146},
  year={2024}
}

@inproceedings{zeng2016convergent,
  title={A convergent fixed-point proximity algorithm accelerated by FISTA for the $\ell_0$ sparse recovery problem},
  author={Zeng, Xueying and Shen, Lixin and Xu, Yuesheng},
  booktitle={International Conference on Imaging, Vision and Learning based on Optimization and PDEs},
  pages={27--45},
  year={2016},
  organization={Springer}
}

@article{xu2023sparse,
  title={Sparse regularization with the $\ell_0$ norm},
  author={Xu, Yuesheng},
  journal={Analysis and Applications},
  volume={21},
  number={04},
  pages={901--929},
  year={2023},
  publisher={World Scientific}
}

@article{wu2022inverting,
  title={Inverting incomplete Fourier transforms by a sparse regularization model and applications in seismic wavefield modeling},
  author={Wu, Tingting and Xu, Yuesheng},
  journal={Journal of Scientific Computing},
  volume={92},
  number={2},
  pages={48},
  year={2022},
  publisher={Springer}
}

@article{attouch2010proximal,
  title={Proximal alternating minimization and projection methods for nonconvex problems: An approach based on the Kurdyka-{\L}ojasiewicz inequality},
  author={Attouch, H{\'e}dy and Bolte, J{\'e}r{\^o}me and Redont, Patrick and Soubeyran, Antoine},
  journal={Mathematics of operations research},
  volume={35},
  number={2},
  pages={438--457},
  year={2010},
  publisher={INFORMS}
}

@inproceedings{he2016deep,
  title={Deep residual learning for image recognition},
  author={He, Kaiming and Zhang, Xiangyu and Ren, Shaoqing and Sun, Jian},
  booktitle={Proceedings of the IEEE conference on computer vision and pattern recognition},
  pages={770--778},
  year={2016}
}

@inproceedings{rahaman2019spectral,
  title={On the spectral bias of neural networks},
  author={Rahaman, Nasim and Baratin, Aristide and Arpit, Devansh and Draxler, Felix and Lin, Min and Hamprecht, Fred and Bengio, Yoshua and Courville, Aaron},
  booktitle={International conference on machine learning},
  pages={5301--5310},
  year={2019},
  organization={PMLR}
}

@article{bao2016image,
  title={Image restoration by minimizing zero norm of wavelet frame coefficients},
  author={Bao, Chenglong and Dong, Bin and Hou, Likun and Shen, Zuowei and Zhang, Xiaoqun and Zhang, Xue},
  journal={Inverse problems},
  volume={32},
  number={11},
  pages={115004},
  year={2016},
  publisher={IOP Publishing}
}

@article{zhang2010nearly,
 ISSN = {00905364, 21688966},
 URL = {http://www.jstor.org/stable/25662264},
 author = {Cun-Hui Zhang},
 journal = {The Annals of Statistics},
 number = {2},
 pages = {894--942},
 publisher = {Institute of Mathematical Statistics},
 title = {NEARLY UNBIASED VARIABLE SELECTION UNDER MINIMAX CONCAVE PENALTY},
 urldate = {2025-09-15},
 volume = {38},
 year = {2010}
}

@article{fan2001variable,
  title={Variable selection via nonconcave penalized likelihood and its oracle properties},
  author={Fan, Jianqing and Li, Runze},
  journal={Journal of the American statistical Association},
  volume={96},
  number={456},
  pages={1348--1360},
  year={2001},
  publisher={Taylor \& Francis}
}

@article{luo2025neurtv,
  title={Neurtv: Total variation on the neural domain},
  author={Luo, Yisi and Zhao, Xile and Ye, Kai and Meng, Deyu},
  journal={SIAM Journal on Imaging Sciences},
  volume={18},
  number={2},
  pages={1101--1140},
  year={2025},
  publisher={SIAM}
}

@inproceedings{zhao2025adaptive,
  title={Adaptive Wavelet-Positional Encoding for High-Frequency Information Learning in Implicit Neural Representation},
  author={Zhao, Hongxu and Gao, Zelin and Wang, Yue and Xiong, Rong and Zhang, Yu},
  booktitle={Proceedings of the AAAI Conference on Artificial Intelligence},
  volume={39},
  number={10},
  pages={10430--10438},
  year={2025}
}

@InProceedings{pmlr-v235-pal24a,
  title = 	 {Implicit Representations via Operator Learning},
  author =       {Pal, Sourav and Adepu, Harshavardhan and Wang, Clinton and Golland, Polina and Singh, Vikas},
  booktitle = 	 {Proceedings of the 41st International Conference on Machine Learning},
  pages = 	 {39022--39041},
  year = 	 {2024},
  editor = 	 {Salakhutdinov, Ruslan and Kolter, Zico and Heller, Katherine and Weller, Adrian and Oliver, Nuria and Scarlett, Jonathan and Berkenkamp, Felix},
  volume = 	 {235},
  series = 	 {Proceedings of Machine Learning Research},
  month = 	 {21--27 Jul},
  publisher =    {PMLR},
  url = 	 {https://proceedings.mlr.press/v235/pal24a.html}
}

@article{lee2021meta,
  title={Meta-learning sparse implicit neural representations},
  author={Lee, Jaeho and Tack, Jihoon and Lee, Namhoon and Shin, Jinwoo},
  journal={Advances in Neural Information Processing Systems},
  volume={34},
  pages={11769--11780},
  year={2021}
}

@article{luo2024revisiting,
  title={Revisiting nonlocal self-similarity from continuous representation},
  author={Luo, Yisi and Zhao, Xile and Meng, Deyu},
  journal={IEEE Transactions on Pattern Analysis and Machine Intelligence},
  year={2024},
  publisher={IEEE}
}

@inproceedings{Shekarforoush2022ResidualMFN,
  title     = {Residual Multiplicative Filter Networks for Multiscale Reconstruction},
  author    = {Shayan Shekarforoush and David B. Lindell and David J. Fleet and Marcus A. Brubaker},
  booktitle = {Advances in Neural Information Processing Systems (NeurIPS)},
  year      = {2022},
  note      = {ArXiv preprint available: arXiv:2206.00746}
}

@article{hertz2021sape,
  title={Sape: Spatially-adaptive progressive encoding for neural optimization},
  author={Hertz, Amir and Perel, Or and Giryes, Raja and Sorkine-Hornung, Olga and Cohen-Or, Daniel},
  journal={Advances in Neural Information Processing Systems},
  volume={34},
  pages={8820--8832},
  year={2021}
}

@inproceedings{Park2019DeepSDF,
  title     = {DeepSDF: Learning Continuous Signed Distance Functions for Shape Representation},
  author    = {Jeong Joon Park and Peter Florence and Julian Straub and Richard Newcombe and Steven Lovegrove},
  booktitle = {Proceedings of the IEEE/CVF Conference on Computer Vision and Pattern Recognition (CVPR)},
  year      = {2019},
  month     = {June},
  pages     = {165--174}
}

@inproceedings{Mescheder2019OccupancyNetworks,
  title     = {Occupancy Networks: Learning 3D Reconstruction in Function Space},
  author    = {Lars Mescheder and Michael Oechsle and Michael Niemeyer and Sebastian Nowozin and Andreas Geiger},
  booktitle = {Proceedings of the IEEE/CVF Conference on Computer Vision and Pattern Recognition (CVPR)},
  year      = {2019},
  month     = {June},
  pages     = {4460--4470}
}

@inproceedings{Lindell2022BACON,
  title     = {BACON: Band-Limited Coordinate Networks for Multiscale Scene Representation},
  author    = {David B. Lindell and Dave Van Veen and Jeong Joon Park and Gordon Wetzstein},
  booktitle = {Proceedings of the IEEE/CVF Conference on Computer Vision and Pattern Recognition (CVPR)},
  year      = {2022},
  month     = {June},
  pages     = {16252--16262}
}

@inproceedings{Fathony2021MFN,
  title     = {Multiplicative Filter Networks},
  author    = {Rizal Fathony and Anit Kumar Sahu and Devin Willmott and J. Zico Kolter},
  booktitle = {International Conference on Learning Representations (ICLR)},
  year      = {2021},
  note      = {OpenReview: \url{https://openreview.net/forum?id=OmtmcPkkhT}}
}

@article{Martel2021ACORN,
  title   = {ACORN: Adaptive Coordinate Networks for Neural Scene Representation},
  author  = {Julien N. P. Martel and David B. Lindell and Connor Z. Lin and Eric R. Chan and Marco Monteiro and Gordon Wetzstein},
  journal = {ACM Transactions on Graphics (Proc. SIGGRAPH)},
  year    = {2021},
  volume  = {40},
  number  = {4},
  pages   = {58:1--58:13},
  doi     = {10.1145/3450626.3459785}
}

@inproceedings{buades2005non,
  title={A non-local algorithm for image denoising},
  author={Buades, Antoni and Coll, Bartomeu and Morel, J-M},
  booktitle={2005 IEEE computer society conference on computer vision and pattern recognition (CVPR'05)},
  volume={2},
  pages={60--65},
  year={2005},
  organization={Ieee}
}

@inproceedings{Zamir2021Restormer,
    title={Restormer: Efficient Transformer for High-Resolution Image Restoration}, 
    author={Syed Waqas Zamir and Aditya Arora and Salman Khan and Munawar Hayat 
            and Fahad Shahbaz Khan and Ming-Hsuan Yang},
    booktitle={CVPR},
    year={2022}
}

@article{
doi:10.1073/pnas.1907377117,
author = {Vegard Antun  and Francesco Renna  and Clarice Poon  and Ben Adcock  and Anders C. Hansen },
title = {On instabilities of deep learning in image reconstruction and the potential costs of AI},
journal = {Proceedings of the National Academy of Sciences},
volume = {117},
number = {48},
pages = {30088-30095},
year = {2020},
doi = {10.1073/pnas.1907377117},
URL = {https://www.pnas.org/doi/abs/10.1073/pnas.1907377117},
eprint = {https://www.pnas.org/doi/pdf/10.1073/pnas.1907377117}}

@article{bhadra2021hallucinations,
  title={On hallucinations in tomographic image reconstruction},
  author={Bhadra, Sayantan and Kelkar, Varun A and Brooks, Frank J and Anastasio, Mark A},
  journal={IEEE transactions on medical imaging},
  volume={40},
  number={11},
  pages={3249--3260},
  year={2021},
  publisher={IEEE}
}

@article{rudin1992nonlinear,
  title={Nonlinear total variation based noise removal algorithms},
  author={Rudin, Leonid I and Osher, Stanley and Fatemi, Emad},
  journal={Physica D: nonlinear phenomena},
  volume={60},
  number={1-4},
  pages={259--268},
  year={1992},
  publisher={Elsevier}
}

@article{keys2003cubic,
  title={Cubic convolution interpolation for digital image processing},
  author={Keys, Robert},
  journal={IEEE transactions on acoustics, speech, and signal processing},
  volume={29},
  number={6},
  pages={1153--1160},
  year={1981},
  publisher={Ieee}
}

@book{10.7551/mitpress/2946.001.0001,
    author = {Wiener, Norbert},
    title = {Extrapolation, Interpolation, and Smoothing of Stationary Time Series: With Engineering Applications},
    publisher = {The MIT Press},
    year = {1949},
    address = {Cambridge, MA},
    doi ={10.7551/mitpress/2946.001.0001},
}

@article{dabov2007image,
  title={Image denoising by sparse 3-D transform-domain collaborative filtering},
  author={Dabov, Kostadin and Foi, Alessandro and Katkovnik, Vladimir and Egiazarian, Karen},
  journal={IEEE Transactions on image processing},
  volume={16},
  number={8},
  pages={2080--2095},
  year={2007},
  publisher={IEEE}
}

@article{telea2004image,
  title={An image inpainting technique based on the fast marching method},
  author={Telea, Alexandru},
  journal={Journal of graphics tools},
  volume={9},
  number={1},
  pages={23--34},
  year={2004},
  publisher={Taylor \& Francis}
}

@article{tancik2020fourier,
  title={Fourier features let networks learn high frequency functions in low dimensional domains},
  author={Tancik, Matthew and Srinivasan, Pratul and Mildenhall, Ben and Fridovich-Keil, Sara and Raghavan, Nithin and Singhal, Utkarsh and Ramamoorthi, Ravi and Barron, Jonathan and Ng, Ren},
  journal={Advances in neural information processing systems},
  volume={33},
  pages={7537--7547},
  year={2020}
}

@inproceedings{saragadam2023wire,
  title={Wire: Wavelet implicit neural representations},
  author={Saragadam, Vishwanath and LeJeune, Daniel and Tan, Jasper and Balakrishnan, Guha and Veeraraghavan, Ashok and Baraniuk, Richard G},
  booktitle={Proceedings of the IEEE/CVF Conference on Computer Vision and Pattern Recognition},
  pages={18507--18516},
  year={2023}
}

@inproceedings{ramasinghe2022beyond,
  title={Beyond periodicity: Towards a unifying framework for activations in coordinate-mlps},
  author={Ramasinghe, Sameera and Lucey, Simon},
  booktitle={European Conference on Computer Vision},
  pages={142--158},
  year={2022},
  organization={Springer}
}

@article{hao2022implicit,
  title={Implicit neural representations with levels-of-experts},
  author={Hao, Zekun and Mallya, Arun and Belongie, Serge and Liu, Ming-Yu},
  journal={Advances in Neural Information Processing Systems},
  volume={35},
  pages={2564--2576},
  year={2022}
}

@article{sitzmann2020implicit,
  title={Implicit neural representations with periodic activation functions},
  author={Sitzmann, Vincent and Martel, Julien and Bergman, Alexander and Lindell, David and Wetzstein, Gordon},
  journal={Advances in neural information processing systems},
  volume={33},
  pages={7462--7473},
  year={2020}
}

@inproceedings{ulyanov2018deep,
  title={Deep image prior},
  author={Ulyanov, Dmitry and Vedaldi, Andrea and Lempitsky, Victor},
  booktitle={Proceedings of the IEEE conference on computer vision and pattern recognition},
  pages={9446--9454},
  year={2018}
}

@inproceedings{liang2021swinir,
  title={Swinir: Image restoration using swin transformer},
  author={Liang, Jingyun and Cao, Jiezhang and Sun, Guolei and Zhang, Kai and Van Gool, Luc and Timofte, Radu},
  booktitle={Proceedings of the IEEE/CVF international conference on computer vision},
  pages={1833--1844},
  year={2021}
}

@article{xu2025multi,
  title={Multi-grade deep learning},
  author={Xu, Yuesheng},
  journal={Communications on Applied Mathematics and Computation},
  pages={1--52},
  year={2025},
  publisher={Springer}
}

@article{muller2022instant,
  title={Instant neural graphics primitives with a multiresolution hash encoding},
  author={M{\"u}ller, Thomas and Evans, Alex and Schied, Christoph and Keller, Alexander},
  journal={ACM transactions on graphics (TOG)},
  volume={41},
  number={4},
  pages={1--15},
  year={2022},
  publisher={ACM New York, NY, USA}
}

@article{chambolle2011first,
  title={A first-order primal-dual algorithm for convex problems with applications to imaging},
  author={Chambolle, Antonin and Pock, Thomas},
  journal={Journal of mathematical imaging and vision},
  volume={40},
  number={1},
  pages={120--145},
  year={2011},
  publisher={Springer}
}

@article{beck2009fast,
  title={A fast iterative shrinkage-thresholding algorithm for linear inverse problems},
  author={Beck, Amir and Teboulle, Marc},
  journal={SIAM journal on imaging sciences},
  volume={2},
  number={1},
  pages={183--202},
  year={2009},
  publisher={SIAM}
}

@article{micchelli2011proximity,
  title={Proximity algorithms for image models: denoising},
  author={Micchelli, Charles A and Shen, Lixin and Xu, Yuesheng},
  journal={Inverse Problems},
  volume={27},
  number={4},
  pages={045009},
  year={2011},
  publisher={IOP Publishing}
}

@article{shen2016wavelet,
  title={Wavelet inpainting with the $\ell_0$ sparse regularization},
  author={Shen, Lixin and Xu, Yuesheng and Zeng, Xueying},
  journal={Applied and Computational Harmonic Analysis},
  volume={41},
  number={1},
  pages={26--53},
  year={2016},
  publisher={Elsevier}
}

@article{fang2024inexact,
  title={Inexact fixed-point proximity algorithm for the $\ell_0$ sparse regularization problem},
  author={Fang, Ronglong and Xu, Yuesheng and Yan, Mingsong},
  journal={Journal of Scientific Computing},
  volume={100},
  number={2},
  pages={58},
  year={2024},
  publisher={Springer}
}

@inproceedings{xu2019training,
  title={Training behavior of deep neural network in frequency domain},
  author={Xu, Zhi-Qin John and Zhang, Yaoyu and Xiao, Yanyang},
  booktitle={International Conference on Neural Information Processing},
  pages={264--274},
  year={2019},
  organization={Springer}
}

@inproceedings{pilanci2020neural,
  title={Neural networks are convex regularizers: Exact polynomial-time convex optimization formulations for two-layer networks},
  author={Pilanci, Mert and Ergen, Tolga},
  booktitle={International Conference on Machine Learning},
  pages={7695--7705},
  year={2020},
  organization={PMLR}
}

@article{Shen-Suter-Tripp:JOTA:2019,
  author =       {Lixin Shen and Bruce W. Suter and Erin E. Tripp},
  title =        {Structured Sparsity Promoting Functions},
  journal=       {Journal of Optimization Theory and Applications},
  volume="183",
  number="3",
  pages="386--421",
  year =        {2019}
}

\end{document}